\documentclass[accepted,nohyperref]{article}

\usepackage{fullpage}
\usepackage{parskip}

\usepackage{authblk}

\usepackage{microtype}
\usepackage{graphicx}
\usepackage{subcaption}
\usepackage{booktabs} %

\usepackage{hyperref}

\usepackage{xcolor}
\usepackage{algorithm}
\usepackage[noend]{algpseudocode}

\usepackage{amsmath}
\usepackage{amssymb}
\usepackage{mathtools}
\usepackage{amsthm}

\usepackage[capitalize,noabbrev]{cleveref}

\theoremstyle{plain}
\newtheorem{theorem}{Theorem}[section]

\newtheorem{lemma}[theorem]{Lemma}
\newtheorem{corollary}[theorem]{Corollary}
\theoremstyle{definition}
\newtheorem{definition}[theorem]{Definition}

\newtheorem{remark}[theorem]{Remark}

\definecolor{linkcolor}{RGB}{83,83,182}

\hypersetup{
    colorlinks=true,
    citecolor=linkcolor,
    linkcolor=linkcolor,
    urlcolor=linkcolor
}

\usepackage[backend=biber,natbib=true,style=authoryear,doi=false,isbn=false,url=false,eprint=false,giveninits=true,maxbibnames=200,maxcitenames=2,mincitenames=1,uniquename=false,uniquelist=false,dashed=false]{biblatex}

\DeclareSourcemap{
  \maps[datatype=bibtex]{
    \map{
      \step[fieldsource=url, match=\regexp{http://(dx.doi.org/|dl.acm.org/)}, final]
      \step[fieldset=url, null]
      \step[fieldset=urldate, null]
    }
  }
}
\title{Differentially Private Coordinate Descent \\
  for Composite Empirical Risk Minimization}

\date{}

\author[1]{Paul Mangold}
\author[1]{Aurélien Bellet}
\author[2,3]{Joseph Salmon}
\author[4]{Marc Tommasi}

\affil[1]{Univ. Lille, Inria,  CNRS, Centrale Lille, UMR 9189 - CRIStAL, F-59000 Lille, France}
\affil[2]{IMAG, Univ Montpellier, CNRS, Montpellier, France}
\affil[3]{Institut Universitaire de France (IUF)}
\affil[4]{Univ. Lille, CNRS, Inria, Centrale Lille,  UMR 9189 - CRIStAL, F-59000 Lille, France}

\usepackage[utf8]{inputenc} %
\usepackage[T1]{fontenc}    %
\usepackage{url}            %
\usepackage{booktabs}       %
\usepackage{amsfonts}       %
\usepackage{nicefrac}       %
\usepackage{microtype}      %
\usepackage{xcolor}         %

\usepackage{fancyvrb}
\usepackage{enumitem}
\usepackage{graphicx}
\usepackage{makecell} %

\usepackage{caption}
\usepackage{duckuments}
\usepackage{xspace}

\usepackage[colorinlistoftodos,textsize=scriptsize]{todonotes}
\setuptodonotes{color=pink}
\usepackage{marginnote}
\usepackage{zref-savepos}

\usepackage{titlesec}
\titleformat*{\subparagraph}{\itshape}

\newcount\todocount

\newcommand{\paul}[1]{\todo{\textbf{Paul:} #1}}
\newcommand{\aurelien}[1]{\todo[color=green]{\textbf{Aurelien:} #1}}
\newcommand{\marc}[2][]{\todo[color=red!60,#1]{\textbf{Marc:} #2}}
\newcommand{\jo}[2][]{\todo[color=blue!60,#1]{\textbf{Jo:} #2}}
\newcommand{\jrs}[1]{\textcolor{red}{#1}}
\renewcommand{\paul}[1]{}
\renewcommand{\aurelien}[1]{}
\renewcommand{\marc}[2][]{}
\renewcommand{\jo}[2][]{}
\renewcommand{\jrs}[1]{}

\newtheorem{claim}{Claim}

\newcommand*{\sketchproofname}{Sketch of Proof}
\newenvironment{sketchproof}[1][\sketchproofname]{
  \begin{proof}[#1] } {
  \end{proof}
}

\usepackage{aligned-overset}

\theoremstyle{theorem}
\newtheorem{innercustomthm}{Theorem}
\newenvironment{restate-theorem}[1]
{\renewcommand\theinnercustomthm{#1}\innercustomthm}
{\endinnercustomthm}

\newcommand{\norm}[1]{\left \lVert #1 \right \rVert}

\newcommand{\scalar}[2]{\left\langle #1, #2 \right\rangle}
\newcommand{\abs}[1]{\left| #1 \right|}
\newcommand{\card}[1]{\left| #1 \right|}
\DeclareMathOperator{\tr}{tr}
\DeclareMathOperator{\diag}{diag}
\DeclareMathOperator{\prox}{prox}
\DeclareMathOperator{\clip}{clip}
\DeclareMathOperator{\conv}{conv}
\DeclareMathOperator*{\argmin}{arg\,min}

\newcommand{\expec}[2]{\mathbb E_{#1} \hspace{-0.29em} \left[ #2 \right]}
\newcommand{\condexpec}[3]{\mathbb E_{#1} \hspace{-0.29em} \left[ #2 \middle\vert #3 \right]}
\newcommand{\prob}[1]{\textnormal{Pr} \left[ #1 \right]}

\newcommand{\NN}{\mathbb N}
\newcommand{\RR}{\mathbb R}
\newcommand{\cA}{\mathcal A}
\newcommand{\cC}{\mathcal C}
\newcommand{\cD}{\mathcal D}
\newcommand{\cF}{\mathcal F}
\newcommand{\cJ}{\mathcal J}

\newcommand{\cM}{\mathcal M}
\newcommand{\cQ}{\mathcal Q}
\newcommand{\cW}{\mathcal W}
\newcommand{\cX}{\mathcal X}

\newcommand{\bbE}{\mathbb E}

\newcommand{\ie}{{\em i.e.,~}}

\newcommand{\eg}{{\em e.g.,~}}

\newcommand{\wrt}{{\em w.r.t.~}}

\begin{document}

\maketitle

\begin{abstract}
  Machine learning models can leak information about the data used to
  train them. To mitigate this issue, Differentially Private (DP) variants of
  optimization
  algorithms like Stochastic Gradient Descent (DP-SGD) have been
  designed to trade-off utility for privacy in Empirical Risk
  Minimization (ERM) problems. In this paper, we propose
  Differentially Private proximal Coordinate Descent (DP-CD), a new
  method to solve composite DP-ERM problems. We derive utility
  guarantees through a novel theoretical analysis of inexact
  coordinate descent. Our results show that, thanks to larger step
  sizes, DP-CD can exploit imbalance in gradient coordinates to
  outperform DP-SGD. We also prove new lower bounds for composite
  DP-ERM under coordinate-wise regularity assumptions, that are nearly
  matched by DP-CD. For practical implementations, we propose to clip
  gradients using coordinate-wise thresholds that emerge
  from our theory, avoiding costly hyperparameter tuning. Experiments
  on real and synthetic data support our results, and show that DP-CD
  compares favorably with DP-SGD.
\end{abstract}

\section{Introduction}

Machine learning fundamentally relies on the availability of data, which can
be sensitive or confidential.
It is now well-known that preventing learned models from leaking information
about individual training points requires particular attention
\citep{shokri2017Membership}.
A standard approach for training models while provably controlling the amount of
leakage is to solve an empirical risk minimization (ERM) problem
under a differential privacy (DP) constraint \citep{chaudhuri2011Differentially}.
In this work, we aim to design a differentially private algorithm which
approximates the solution to a composite ERM problem of the form:
\begin{align}
  \label{eq:dp-erm}
  w^* \in
  \argmin_{w \in \mathbb{R}^p}
  \left\{
  \frac{1}{n} \sum_{i=1}^n \ell(w; d_i) + \psi(w)
  \right\}
  \enspace,
\end{align}
where $D = (d_1, \dots, d_n)$
is a dataset of $n$ samples drawn from a universe $\cX$,
$\ell: \RR^p \times \cX \rightarrow \RR$ is a loss function which is convex
and smooth in $w$, and
$\psi: \RR^p \rightarrow \RR$ is a convex regularizer which is separable (\ie
$\psi(w) = \sum_{j=1}^p \psi_j(w_j)$) and typically nonsmooth (\eg
$\ell_1$-norm).

Differential privacy constraints induce a trade-off between the privacy and
the utility (i.e., optimization error) of the solution of~\eqref{eq:dp-erm}.
This trade-off was made explicit by \citet{bassily2014Private}, who derived
lower bounds on the achievable error given a fixed privacy budget.
To solve the DP-ERM problem in practice, the most popular approaches are based
on Differentially Private variants of Stochastic Gradient Descent (DP-SGD)
\citep{bassily2014Private,abadi2016Deep,wang2017Differentially}, in which
random perturbations are added to the (stochastic) gradients.
\citet{bassily2014Private} analyzed DP-SGD in the non-smooth DP-ERM setting,
and \citet{wang2017Differentially} then proposed an efficient DP-SVRG
algorithm for composite DP-ERM.
Both algorithms match known lower bounds.
SGD-style algorithms perform well in a wide variety of settings, but
also have some flaws: they either require small (or decreasing) step
sizes or variance reduction schemes to guarantee convergence, and they
can be slow when gradients' coordinates are imbalanced.
These flaws propagate to the private counterparts of these
algorithms.
Despite a few attempts at designing other differentially private solvers for
ERM under different setups
\citep{talwar2015Nearly,damaskinos2021Differentially}, the differentially
private optimization toolbox remains limited, which undoubtedly restricts the
resolution of practical problems.

In this paper, we propose and analyze a Differentially Private proximal
Coordinate
Descent algorithm (DP-CD), which performs updates based on perturbed
coordinate-wise gradients (\ie partial derivatives).  Coordinate
Descent (CD) methods have encountered a large success in non-private
machine learning due to their simplicity and effectiveness
\citep{liu2009Blockwise,friedman2010Regularization,chang2008Coordinate,sardy2000Block},
and have seen a surge of practical and theoretical interest in the
last decade \citep{Nesterov12,wright2015Coordinate,shi2017Primer,
  richtarik2014Iteration,fercoq2014Accelerated,tappenden2016Inexact,
  hanzely2020Variance,nutini2015Coordinate,karimireddy2019Efficient}.
In contrast to SGD, they converge with constant step sizes that adapt to
the coordinate-wise smoothness of
the objective. Additionally, CD updates naturally tend to
have a lower sensitivity. Operating with partial gradients thus enables
our private algorithm to reduce the perturbation required to
guarantee privacy without resorting to
amplification by
subsampling \citep{Balle_subsampling,mironov2019Enyi}.

We propose a novel analysis of proximal CD with perturbed gradients to
derive optimal upper bounds on the privacy-utility trade-off achieved
by DP-CD.
We prove a
recursion on distances of CD iterates to an optimal point that keeps track of
coordinate-wise regularity
constants in a tight manner and allows to use
large, constant step sizes that
yield high utility. Our results highlight the fact that DP-CD
can exploit imbalanced gradient coordinates to outperform DP-SGD.
They also improve upon known convergence rates for inexact CD in the
non-private setting
\citep{tappenden2016Inexact}.
We assess the optimality of DP-CD by deriving lower bounds
that capture coordinate-wise Lipschitz regularity measures, and show that
DP-CD matches those bounds up to logarithmic factors.
Our lower bounds also suggest interesting perspectives for future work on
DP-CD algorithms.

Our theoretical results
have important consequences for practical
implementations, which heavily rely on gradient clipping to achieve good
utility.
In contrast to DP-SGD, DP-CD requires to set \emph{coordinate-wise} clipping
thresholds, which can lead to impractical coordinate-wise hyperparameter tuning.
We instead propose a simple rule for adapting these thresholds from a
single hyperparameter. We also show how the coordinate-wise smoothness
constants used by DP-CD can be
estimated privately. We validate our theory with numerical
experiments on real and synthetic datasets. These experiments further
show that even in balanced problems, DP-CD can still improve over
DP-SGD, confirming the relevance of DP-CD for DP-ERM.

Our main contributions can be summarized as follows:
\begin{enumerate}
  \item We propose the first proximal CD algorithm for composite DP-ERM,
        formally prove its utility, and highlight regimes where it outperforms DP-SGD.
  \item We show matching lower bounds under coordinate-wise regularity
        assumptions.
      \item We give practical guidelines to use DP-CD, and show its
        relevance through numerical experiments.
\end{enumerate}

The rest of this paper is organized as follows.
We first describe some mathematical background in
\Cref{sec:preliminaries}.
In \Cref{sec:diff-priv-coord}, we present our DP-CD algorithm,
show that it satisfies DP, establish utility guarantees, and
compare these guarantees with those of DP-SGD.
In \Cref{sec:utility-lower-bounds}, we derive lower bounds under
coordinate-wise regularity assumptions, and
show that DP-CD can match them. \Cref{sec:dp-cd-practice} discusses practical
questions related to gradient clipping and the private estimation of
smoothness constants.
\Cref{sec:numerical-experiments} presents our numerical experiments,
comparing DP-CD and DP-SGD on LASSO and $\ell_2$-regularized
logistic regression problems. %
Finally, we review existing work in
\Cref{sec:related-works}, and conclude with promising lines of future work in
\Cref{sec:conclusion-and-discussion}.

\section{Preliminaries}
\label{sec:preliminaries}

In this section, we introduce important technical notions that will be used
throughout the paper.

\paragraph{Norms.}
We start by defining two conjugate norms that will be crucial in our analysis,
for they allow to keep track of coordinate-wise quantities.
Let $\scalar{u}{v} = \sum_{j=1}^p u_i v_i$ be the Euclidean dot product, let $M = \diag(M_1, \dots, M_p)$ with $M_1, \dots, M_p > 0$, and
\begin{align*}
  \norm{w}_M = \sqrt{\scalar{Mw}{w}}\enspace,\quad\quad\quad
  \norm{w}_{M^{-1}} = \sqrt{\scalar{M^{-1}w}{w}} \enspace.
\end{align*}
When $M$ is the identity matrix $I$, the $I$-norm $\norm{\cdot}_I$ is the standard $\ell_2$-norm $\norm{\cdot}_2$.

\paragraph{Regularity assumptions.}
We recall classical regularity assumptions along with ones
specific to the coordinate-wise setting.
We denote by $\nabla f$ the gradient of
a differentiable function $f$, and by $\nabla_j f$ its $j$-th coordinate.
We denote by $e_j$ the $j$-th vector of $\RR^p$'s canonical basis.

\textit{Convexity:} a differentiable function $f : \RR^p
  \rightarrow \RR$ is convex if
for all $v, w \in \RR^p$,
$f(w) \ge f(v) + \scalar{\nabla f(v)}{w - v}$.

\textit{Strong convexity:} a differentiable function $f : \RR^p \rightarrow
  \RR$ is
$\mu_M$-strongly-convex \wrt the norm $\smash{\norm{\cdot}_M}$ if
for all $v, w \in \RR^p$,
$f(w) \ge f(v) + \scalar{\nabla f(v)}{w - v} + \frac{\mu_M}{2}\norm{w - v}_M^2$.
The case $M_1=\cdots=M_p=1$ recovers standard $\mu_I$-strong convexity \wrt
the $\ell_2$-norm.

\textit{Component Lipschitzness:} a function $f : \RR^p \rightarrow \RR$
is
$L$-component-Lipschitz for $L = (L_1,\dots,L_p)$ with $L_1,\dots,L_p > 0$ if
for all $w \in \RR^p$, $t \in \RR$ and $j \in [p]$,
$\abs{f(w + t e_j) - f(w)} \le L_j \abs{t}$.
It is $\Lambda$-Lipschitz if for all $v, w \in \RR^p$,
$\abs{f(v) - f(w)} \le \Lambda \norm{v - w}_2$.

\textit{Component smoothness:} a differentiable function $f : \RR^p
  \rightarrow \RR$ is
$M$-component-smooth for $M_1,\dots,M_p > 0$ if
for all $v, w \in \RR^p$,
$f(w) \le f(v) + \scalar{\nabla f(v)}{w - v} + \frac{1}{2}\norm{w - v}_{M}^2$.
When $M_1=\dots=M_p=\beta$, $f$ is said to be $\beta$-smooth.

The above component-wise regularity hypotheses are not restrictive:
$\Lambda$-Lipschitzness
implies $(\Lambda, \dots, \Lambda)$-component-Lipschitzness and
$\beta$-smoothness implies $(\beta, \dots, \beta)$-component-smoothness.
Yet, the actual component-wise constants of a function can be much
lower than what can be deduced from their global counterparts.
This will be crucial for our analysis and in the performance of DP-CD.

\begin{remark}
  \label{rmq:constrained-regularity-assumptions}
  When $\psi$ is the characteristic function of a convex set (with separable
  components), the regularity assumptions only need to hold on this
  set. This allows considering problem~\eqref{eq:dp-erm} with a smooth
  objective under box-constraints.
\end{remark}

\paragraph{Differential privacy (DP).}

Let $\cD$ be a set of datasets and $\cF$ a set of possible outcomes.
Two datasets $D, D' \in \cD$ are said \textit{neighboring}
(denoted by $D \sim D'$) if they differ on at most one element.

\begin{definition}[Differential Privacy, \citealt{dwork2006Differential}]
  A randomized algorithm
  $\cA : \mathcal D
    \rightarrow \mathcal F$ is $(\epsilon, \delta)$-differentially private if,
  for all neighboring datasets $D, D' \in \mathcal D$ and all
  $S \subseteq \mathcal F$ in the range of $\cA$:
  \begin{align*}
    \prob{\cA(D) \in S} \le \exp(\epsilon) \prob{\cA(D') \in S} + \delta \enspace.
  \end{align*}
\end{definition}
The value of a function $h: \mathcal D \rightarrow \mathbb R^p$ can be privately
released using the Gaussian mechanism, which adds centered Gaussian noise
to $h(D)$ before releasing it \citep{dwork2013Algorithmic}.
The scale of the noise is calibrated to the sensitivity $
  \Delta(h)
  = \sup_{D \sim D'} \norm{h(D) - h(D')}_2$ of $h$.
In our setting, we will perturb coordinate-wise gradients: we denote by
$\Delta(\nabla_j \ell)$ the sensitivity of the $j$\nobreakdash-th coordinate
of gradient of the loss function $\ell$ with respect to the data.
When $\ell(\cdot;d)$ is $L$-component-Lipschitz for all $d\in\mathcal{X}$, upper
bounds on these sensitivities are readily available: we have
$\Delta(\nabla_j\ell) \le 2L_j$ for any $j\in[p]$ (see \Cref{sec:lemma-sensitivity}).
The following quantity, relating the coordinate-wise sensitivities of gradients
to coordinate-wise smoothness is central in our analysis:
\begin{align}
  \label{eq:delta-lipschitz-norm}
  \Delta_{M^{-1}}(\nabla \ell)
  = \Big(\sum_{j=1}^p \frac{1}{M_j} \Delta (\nabla_j\ell)^2\Big)^{\frac{1}{2}}
  \leq \! 2 \norm{L}_{M^{-1}}\enspace.
\end{align}
In this paper, we consider the classic central model of DP, where a trusted
curator has access to the raw dataset and releases a model trained on this
dataset\footnote{In fact, our privacy guarantees hold even if all
  intermediate iterates are released (not just the final model).}.

\section{Differentially Private Coordinate Descent}
\label{sec:diff-priv-coord}

In this section, we introduce the
Differentially Private proximal Coordinate Descent (DP-CD) algorithm to
solve problem~\eqref{eq:dp-erm} under $(\epsilon,\delta)$-DP constraints.
We first describe our algorithm, show how to parameterize it to
satisfy the desired privacy constraint, and prove corresponding utility
results.
Finally, we compare these utility guarantees with DP-SGD.

\subsection{Private Proximal Coordinate Descent}

Let $D = \{d_1, \dots, d_n\} \in \cX^n$ be a dataset.
We denote by $f(w) = \frac{1}{n}\sum_{i=1}^n \ell(w; d_i)$ the $M$-component-smooth
part of \eqref{eq:dp-erm},
by $\psi(w) = \sum_{j=1}^p \psi_j(w_j)$ its separable part,
and let $F(w) = f(w) + \psi(w)$.
Proximal coordinate descent methods \cite{richtarik2014Iteration}
solve problem~\eqref{eq:dp-erm} through iterative proximal gradient
steps along each coordinate of~$F$.
Formally, given $w \in \RR^p$ and $j \in [p]$, the $j$-th coordinate of
$w$ is updated as follows:
\begin{align}
  \label{eq:proximal-update-nonoise}
  w_j^+ = \prox_{\gamma_j\psi_j} \big(w_j - \gamma_j
  \nabla_j f(w_t)\big)\enspace,
\end{align}
where $\gamma_j>0$ is the step size and $\prox_{\gamma_j\psi_j}
(w)= \smash{\argmin_{v\in\RR^p} \big
\{ \frac{1}{2} \norm{v - w}_2^2 + \gamma_j\psi_j(v) \big\}}$ is the proximal
operator associated with $\psi_j$ \citep{parikh2014Proximal}.

\begin{algorithm*}[t]
  \caption{Differentially Private Proximal Coordinate Descent Algorithm
    (DP-CD).}
  \label{algo:dp-cd}
  \textbf{Input}:
  noise scales $\sigma = (\sigma_1, \dots, \sigma_p)$ for $\sigma_1,\dots,\sigma_p > 0$;
  step sizes $\gamma_1,\dots,\gamma_p > 0$;
  initial point $\bar w^0 \in \mathbb{R}^p$;
  iteration budgets $T, K > 0$.
  \begin{algorithmic}[1]
    \For{$t = 0, \dots, T-1$}
    \State Set $\theta^0 = \bar w^t$
    \For{$k = 0, \dots, K-1$}
    \State Pick $j$ from $\{1, \dots, p\}$ uniformly at random
    \State Draw $\eta_j \sim \mathcal N(0, \sigma_j^2)$
    \label{algo-line:noise-generation}
    \State Set $\theta^{k+1} = \theta^k$
    \State Set $\theta^{k+1}_{j} = \prox_{\gamma_{j}\psi_{j}} (\theta^{k}_{j} -
      \gamma_{j} (\nabla_{j} f(\theta^k) + \eta_j))$
    \label{algo-line:coordinate-minimization-update}
    \vspace*{-.02cm}
    \EndFor
    \State Set $\bar w_{t+1} = \frac 1K \sum_{k=1}^K \theta^k$
    \label{algo-line:periodic_avg}
    \EndFor
    \State \Return $ w_{priv} = \bar w_T$
  \end{algorithmic}
\end{algorithm*}
Update \eqref{eq:proximal-update-nonoise} only requires the computation of the
$j$-th entry of the gradient. To satisfy differential privacy, we perturb this
gradient entry with additive Gaussian noise of variance $\sigma_j^2$.
The complete DP-CD procedure is shown in \Cref{algo:dp-cd}.
At each iteration, we pick a coordinate uniformly at random
and update according to~\eqref{eq:proximal-update-nonoise}, albeit with noise
addition (see line \ref{algo-line:coordinate-minimization-update}).
For technical reasons related to our analysis, we use a
periodic averaging scheme (line~\ref{algo-line:periodic_avg}).
This scheme is similar to DP-SVRG \citep{johnson2013Accelerating}, although
no variance reduction is required since DP-CD computes coordinate gradients
over the whole dataset.

\subsection{Privacy Guarantees}
\label{sec:privacy-guarantees}

For \Cref{algo:dp-cd} to satisfy $(\epsilon,\delta)$-DP, the noise scales
$\sigma=(\sigma_1,\dots,\sigma_p)$ can be calibrated as given in \Cref
{thm:dp-cd-privacy}. %

\begin{theorem}
  \label{thm:dp-cd-privacy}
  Assume $\ell(\cdot;d)$ is $L$-component-Lipschitz $\forall d\in\cX$.
  Let $\epsilon \leq 1$ and $\delta < 1/3$.
  If $\sigma_j^2 = \frac{12L_j^2 TK \log(1/\delta)}{n^2\epsilon^2}$
  for all $j \in [p]$,
  then \Cref{algo:dp-cd} satisfies $(\epsilon, \delta)$-DP.
\end{theorem}
\begin{sketchproof}(Complete proof in \Cref{sec:proof-privacy}).
  We track the privacy loss using Rényi differential privacy (RDP),
  which gives better guarantees than $(\epsilon,\delta)$-DP for the
  composition
  of Gaussian mechanisms \citep{mironov2017Renyi}.  The
  $j$\nobreakdash-th entry of $\nabla f$ has sensitivity
  $\Delta(\nabla_j f) = {\Delta(\nabla_j \ell)}/{n} \le {2L_j}/{n}$.
  By the Gaussian mechanism each iteration of DP-CD is
  $(\alpha, \frac{2 L_j^2 \alpha}{n^2\sigma_j^2})$-RDP for all
  $\alpha > 1$. The composition theorem for RDP gives a global RDP
  guarantee for DP-CD, that we convert to $(\epsilon,\delta)$-DP using
  Proposition~3 of \citet{mironov2017Renyi}. Choosing $\alpha$
  carefully finally proves the result.
\end{sketchproof}

The dependence of the noise scales on $\epsilon$, $\delta$, $n$ and
$TK$ (the number of updates) in \Cref{thm:dp-cd-privacy} is standard
in DP-ERM. However, the noise is calibrated to the loss function's
\emph{component}-Lipschitz constants. These can be much lower their
global counterpart, the latter being used to calibrate the noise
in DP-SGD algorithms. This will be crucial for DP-CD to achieve better utility
than DP-SGD in some regimes.
We also note that, unlike DP-SGD, DP-CD does not rely on privacy
amplification by subsampling \citep{Balle_subsampling,mironov2019Enyi}, and
thereby avoids the approximations required by these
schemes to bound the privacy loss.

\begin{remark}
  Theorem~\ref{thm:dp-cd-privacy} assumes $\epsilon \in (0,1]$ to give
  a simple closed form for the noise scales. In practice we compute
  tighter values numerically using Rényi DP formulas directly (see
  Eq.~\ref{eq:full_privacy_formula} in \Cref{sec:proof-privacy}),
  removing the need for this assumption.
\end{remark}

\subsection{Utility Guarantees}
\label{sec:utiltity-analysis-cd}

We now state our central result on the utility of DP-CD
for the composite DP-ERM problem.
As done in previous work, we use the asymptotic notation $\widetilde O$ to hide
non-significant logarithmic factors. Non-asymptotic utility bounds can be
found in \Cref {sec-app:proof-utility}.
\begin{theorem}
  \label{thm:cd-utility}
  Let $\ell(\cdot; d)$ be a convex and $L$-component-Lipschitz loss
  function for all $d \in \cX$, and $f$ be convex and
  $M$-component-smooth. Let
  $\psi : \RR^p \rightarrow \RR$ be a convex and separable function.
  Let $\epsilon \leq 1, \delta < 1/3$ be the privacy budget.  Let $w^*$
  be a minimizer of $F$ and $F^* = F(w^*)$.
  Let $w_{priv}\in\mathbb{R}^p$ be the output of \Cref{algo:dp-cd} with step
  sizes $\gamma_j = {1}/{M_j}$, and noise
  scales $\sigma_1,\dots,\sigma_p$ set as in Theorem~\ref{thm:dp-cd-privacy} (with $T$ and $K$ chosen below) to ensure
  $(\epsilon,\delta)$-DP.
  Then, the following holds:
  \begin{enumerate}[leftmargin=12pt]
    \item For $F$ convex, $K=O\left( \frac{R_M \sqrt{p} n \epsilon}{\norm{L}_{M^{-1}}} \right)$, and $T = 1$, then:
          \begin{align*}
            \expec{}{F(w_{priv}) - F^*}
            = \widetilde O\bigg(\frac{\sqrt{p \log(1/\delta)}}{n\epsilon}
            \norm{L}_{M^{-1}} R_M\bigg)\enspace,
          \end{align*}
          where $R_M = \max(\sqrt{F(w^0) - F(w^*)}, \norm{w^0 - w^*}_M)$
          and more simply $R_M = \norm{w^0 - w^*}_M$ when $\psi = 0$.
    \item For $F$ $\mu_M$-strongly convex w.r.t. $\smash{\norm{\cdot}_M}$,
          $K = O\left(p/\mu_M\right)$, and
          $T = O\left( \log(n\epsilon \mu_M/p \norm{L}_{M^{-1}}) \right)$, then:
          \begin{align*}
            \expec{}{F(w_{priv}) - F^*}
            = \widetilde O\bigg(\frac{p\log(1/\delta)}{n^2 \epsilon^2}
            \frac{\norm{L}_{M^{-1}}^2}{\mu_M}
            \bigg)\enspace.
          \end{align*}
  \end{enumerate}%
  Expectations are over the randomness of the algorithm.
\end{theorem}
\begin{sketchproof}(Complete proof in \Cref{sec-app:proof-utility}).
  Existing analyses of CD fail to track the noise tightly
  across coordinates when adapted to the private setting. Contrary to
  these classical analyses, we prove a recursion on
  $\mathbb{E}\norm{\theta^k - w^*}_M^2$, rather than on
  $\expec{}{F(\theta^{k}) - F(w^*)}$. Our key technical result is a descent
  lemma
  (\Cref{lemma:descent-lemma}) allowing us to obtain
  \begin{align}
    \label{eq:sketch-proof:cd-utility:first-ineq}
     & \expec{}{F(\theta^{k+1}) - F^*} - \frac{p - 1}{p} \expec{}{F(\theta^k) - F^*}
       \le \mathbb{E}\norm{\theta^k - w^*}_M^2 - \mathbb{E}\norm{\theta^{k+1} - w^*}_M^2
       + \frac{\norm{\sigma}_M^2}{p}
       \enspace.
  \end{align}
  The above inequality shows that coordinate-wise updates leave a fraction
  $\frac{p-1}{p}$ of the function ``unchanged'', while the remaining
  part decreases (up to additive noise). Importantly, all quantities are
  measured
  in $M$-norm. When summing
  \eqref{eq:sketch-proof:cd-utility:first-ineq} for $k=0,\dots,K-1$,
  its left hand side simplifies and its right hand side is simplified
  as a telescoping sum:
  \begin{align}
    \label{eq:sketch-proof:cd-utility:second-ineq}
     & \frac{1}{p}\sum_{k=1}^{K} \expec{}{F(\theta^{k}) - F^*}
      \le \expec{}{F(\bar w^t) - F^*} + \mathbb{E}\norm{\bar w^t - w^*}_M^2
    + \frac{K}{p}\norm{\sigma}_{M^{-1}}^2\enspace,
  \end{align}
  where $\bar w^t$ comes from $\theta^0 = \bar w^t$. As
  $\bar w^{t+1} = \sum_{k=1}^K \frac{\theta^k}{K}$ and $F$ is convex,
  we have
  $F(\bar w^{t+1}) - F^* \le \frac{1}{K} \sum_{k=1}^K F(\theta^k) -
    F^*$. This proves the sub-linear convergence (up to an additive
  noise term) of the inner loop. The result in the convex case follows
  directly (since $T=1$, only one inner loop is run).  For strongly
  convex $F$, it further holds that
  $\mathbb{E}\norm{\bar w^t - w^*}_M^2 \le
    \frac{2}{\mu_M}\expec{}{F(\bar w^t) - F(w^*)}$.  Replacing in
  \eqref{eq:sketch-proof:cd-utility:second-ineq} with large enough $K$
  gives
  $\expec{}{F(\bar w^{t+1}) - F^*} \le \tfrac{1}{2} \expec{}{F(\bar
      w^{t}) - F^*} + \norm{\sigma}_{M^{-1}}^2,$ and linear convergence
  (up to an additive noise term) follows. Finally, $K$ and $T$ are chosen to
  balance the ``optimization'' and the ``privacy'' errors.
\end{sketchproof}

\begin{remark}
  \label{rmq:improvement-inexact-coordinate-descent}
  Our novel convergence proof of CD is also useful in the non-private
  setting. In particular, we improve upon known convergence rates for
  inexact CD methods with additive error \citep{tappenden2016Inexact},
  under the hypothesis that gradients are noisy and unbiased. In their
  formalism, we have $\alpha = 0$ and
  $\beta = \norm{\sigma}_{M^{-1}}^2/p$. With our analysis, the
  algorithm requires $2pR_M^2 / (\xi - p\beta)$ (resp.
  $4p/\mu_M \log((F(w^0) - F^*) / (\xi - p\beta))$) iterations to
  achieve expected precision $\xi > p\beta$ when $F$ is convex
  (resp. $\mu_M$-strongly-convex \wrt $\norm{\cdot}_M$), improving
  upon \citet{tappenden2016Inexact}'s results by a factor
  $\sqrt{p\beta / 2R_M^2}$ (resp. $\mu_M/2$). See \Cref
  {sec:proof-remark-1} for details.  Moreover, unlike this prior work,
  our analysis does not require the objective to decrease at each
  iteration, which is essential to guarantee DP.
\end{remark}

Our utility guarantees stated in \Cref{thm:cd-utility} directly depend on
precise coordinate-wise regularity measures of the objective function.
In particular, the initial distance to optimal, the strong convexity parameter
and the overall sensitivity of the loss function are measured in the norms
$\smash{\norm{\cdot}_M}$ and $\smash{\norm{\cdot}_{M^{-1}}}$
(\ie weighted by coordinate-wise smoothness constants or their inverse).
In the remainder of this section, we thoroughly compare our utility results
with existing ones for DP-SGD.
We will show the optimality of our utility guarantees in
Section~\ref{sec:utility-lower-bounds}.

\begin{table*}[t]
  \centering
  \caption{
    Utility guarantees for DP-CD, DP-SGD, and DP-SVRG for $L$-component-Lipschitz, $\Lambda$-Lipschitz loss.
  } \label{table:utility-cd-sgd}
    \begin{tabular*}{\textwidth}{c @{\extracolsep{\fill}} c c c}
      \toprule
      & Convex
      & Strongly-convex \\
      \midrule
      DP-CD (this paper)
      & $\displaystyle \widetilde O\left(\frac{\sqrt{p \log(1/\delta)}}{n\epsilon}\norm{L}_{M^{-1}} R_{ M}\right)$
      & $\displaystyle \widetilde O\left(\frac{p \log(1/\delta) }{n^2 \epsilon^2}\frac{\norm{L}_{M^{-1}}^2}{\mu_{ M}} \right)$\\
      \midrule
      \makecell{DP-SGD \citep{bassily2014Private} \\ DP-SVRG \citep{wang2017Differentially}}
      & $\displaystyle \widetilde O\left(\frac{\sqrt{p\log(1/\delta)}}{n\epsilon}\Lambda R_I\right)$
      & $\displaystyle \widetilde O\left(\frac{p \log(1/\delta)}{n^2 \epsilon^2} \frac{ \Lambda^2}{\mu_I}\right)$ \\
      \bottomrule
    \end{tabular*}
  \vskip -0.1in
\end{table*}

\subsection{Comparison with DP-SGD and DP-SVRG}
\label{sec:comparison-with-dp-sgd}

We now compare DP-CD with DP-SGD and DP-SVRG, for which
\citet{bassily2014Private} and \citet{wang2017Differentially} proved utility
guarantees.
In this section, we assume that the loss function $\ell$ satisfies the hypotheses
of \Cref{thm:cd-utility}, and is $\Lambda$-Lipschitz.
We denote by $\mu_I$ the strong convexity parameter
of $\ell(\cdot, d)$ \wrt $\norm{\cdot}_2$ and $R_I$ the equivalent of $R_M$ when
$M$ is the identity matrix $I$.
As can be seen from \Cref{table:utility-cd-sgd}, comparing
DP-CD and DP-SGD boils down to comparing
$\smash{\norm{L}_{M^{-1}} R_{M}}$ with $\Lambda R_I$ for
convex functions and ${\norm{L}_{M^{-1}}^2 }/\mu_M$
with ${\Lambda^2 }/{\mu_I}$ for strongly-convex functions.
We compare these terms in two scenarios, depending on the distribution of
coordinate-wise smoothness constants.
To ease the comparison, we assume that $\smash{R_M = \norm{w^0 - w^*}_M}$ and $\smash{R_I = \norm{w^0 - w^*}_I}$ (which is notably the case when $\psi = 0$), and that $F$ has a unique minimizer $w^*$.

\paragraph{Balanced.}
When the smoothness constants $M$ are all equal,
$\norm{L}_{M^{-1}} R_{ M} = \norm{L}_{2} R_I$ and
${\norm{L}_{M^{-1}}^2 }/{\mu_M} = {\norm{L}_{2}^2 }/{\mu_I}$.  This
boils down to comparing $\norm{L}_{2}$ to $\Lambda$. As
$\Lambda \le \norm{L}_{2} \le \sqrt{p}\Lambda$, DP-CD can be up to $p$
times worse than DP-SGD. This can only happen when features are extremely
correlated, which is generally not the case in machine learning.  We
show empirically in \Cref{sec:stand-datas} that, even in balanced regimes,
DP-CD can still significantly outperform DP-SGD.
\paul{is it really p times worse?}

\paragraph{Unbalanced.}
More favorable regimes exist when smoothness constants are imbalanced.
To illustrate this, consider the case where the first
coordinate of the loss function $\ell$ dominates others.
There, $M_{\max} \!=\! M_1 \!\gg\! M_{\min} \!=\! M_j$ and
$L_{\max} \!=\! L_1 \!\gg\! L_{\min}\!=\! L_j $ for all $j\neq 1$, so that
$L_1^2/M_1$ dominates the other terms of $\norm{L}_{M^{-1}}^2$.  This
yields
$\norm{L}_{M^{-1}}^2 \approx L_1^2 / M_1 \approx \Lambda / M_{\max}$,
and $\mu_M = \mu_I M_{\min}$.
Moreover, if the first coordinate of $w^*$ is already well estimated
by $w^0$ (which is common for sparse models), then
$R_M \approx M_{\min}
  R_I$. %
We obtain that
$\norm{L}_{M^{-1}} R_M \approx \sqrt{{M_{\min}}/{M_{\max}}} \Lambda
  R_I$ for convex losses and
$\frac{\norm{L}_{M^{-1}}}{\mu_M} \approx
  \frac{M_{\min}}{M_{\max}}\frac{\Lambda^2}{\mu_I}$ for strongly-convex
ones.
In both cases, DP-CD can perform arbitrarily better than DP-SGD,
depending on the ratio between the smallest and largest
coordinate-wise smoothness constants of the loss function.  This is
due to the inability of DP-SGD to adapt its step size to each coordinate.
DP-CD thus converges quicker than DP-SGD on coordinates with
smaller-scale gradients, requiring fewer accesses to the dataset, and
in turn less noise addition. We give more details on this comparison
in \Cref{sec-app:comparison-with-dp}, and complement it with an empirical
evaluation on synthetic and real-world data in
Section~\ref{sec:numerical-experiments}.

\section{Lower Bounds}
\label{sec:utility-lower-bounds}

We now prove a new lower bound on the error achievable for composite DP-ERM
with $L$-component-Lipschitz loss
functions. While our proof borrows some ideas from the lower bounds known for
constrained ERM with $\Lambda$-Lipschitz losses \citep{bassily2014Private},
deriving our lower
bounds requires to address a number of specific challenges.
First, we cannot use an $\ell_2$ norm constraint as in
\citet{bassily2014Private} in the design of the worst-case problem instances:
we can only
rely on \emph{separable} regularizers. Second, imbalanced coordinate-wise
Lipschitz constants prevent lower-bounding the distance between an arbitrary
point and the solution. This leads us to revisit
the construction of a ``reidentifiable dataset'' from
\citet{bun2014Fingerprinting} so that we
have $L$-component-Lipschitzness while the sum of each
column is large enough, which is crucial in our proof. The full proof is given
in \Cref
{sec:utility-lower-bounds-1}.
\begin{theorem}
  \label{thm:utility-lower-bounds}
  Let $n, p > 0$, $\epsilon > 0$, $\delta = o(\frac{1}{n})$,
  $L_1, \dots, L_p > 0$, such that for all $\cJ \subseteq [p]$ of size at least
  $\lceil \frac{p}{75} \rceil$, $\sum_{j\in\cJ} L_j^2 = \Omega(\norm{L}_2^2)$.
  Let $\cX = \prod_{j=1}^p \{\pm L_j\}$ and consider any
  $(\epsilon, \delta)$-differentially private algorithm that outputs $w^{priv}$.
  In each of the two following cases there exists a dataset $D \in \cX^n$,
  a $L$-component-Lipschitz loss $\ell(\cdot, d)$ for all $d \in D$ and a
  regularizer $\psi$ so that, with $F$ the objective of~\eqref{eq:dp-erm}
  minimal at $w^* \in \RR^p$:
  \begin{enumerate}
    \item If $F$ is convex:
          \begin{align*}
            \expec{}{F(w^{priv};D) - F(w^*)} = \Omega\Big( \frac{\sqrt{p} \norm{L}_2
              \norm{w^*}_2}{n\epsilon} \Big)\enspace.
          \end{align*}
    \item If $F$ is $\mu_I$-strongly-convex \wrt $\norm{\cdot}_2$:
          \begin{align*}
            \expec{}{F(w^{priv};D) - F(w^*)} = \Omega\Big( \frac{p \norm{L}_2^2}{\mu_I
              n^2\epsilon^2} \Big)\enspace.
          \end{align*}
  \end{enumerate}
\end{theorem}
We recover the lower bounds of \citet{bassily2014Private} for
$\Lambda$-Lipschitz
losses as a special case of ours by setting $L_1 = \cdots = L_p = {\Lambda}/{
  \sqrt{p}}$. In this case, the
loss function used in our proof is indeed $(\sum_{j=1}^p L_j^2)^{1/2}
  =\Lambda$-Lipschitz.
To relate these lower bounds to the performance of DP-CD, consider a
suboptimal version of
our algorithm where the step sizes are set to
$\gamma_1 = \cdots = \gamma_p = ({\max_j M_j})^{-1}$.
In this setting, results from \Cref{thm:cd-utility} still hold, and match the
lower bounds from \Cref{thm:utility-lower-bounds} up to logarithmic factors.
We leave open the question of the optimality of DP-CD under the additional
hypothesis of smoothness.

We note that the assumption on the sum of the $L_j$'s over a set of indices
$\cJ$ in \Cref{thm:utility-lower-bounds} can be eliminated at the cost of an
additional
factor of ${L_{\min}}/{L_{\max}}$ for convex losses and
$({L_{\min}}/{L_{\max}})^2$ for strongly-convex losses, making the bound looser.
Although the aforementioned assumption may seem solely technical, we
conjecture that better utility is possible when a few
coordinate-wise Lipschitz constants dominate the others.
We discuss this further in \Cref{sec:conclusion-and-discussion}.

\section{DP-CD in Practice}
\label{sec:dp-cd-practice}

We now discuss practical questions related to DP-CD. First, we show how
to implement coordinate-wise gradient clipping using a single hyperparameter.
Second, we explain how to privately estimate the smoothness constants.
Finally, we
discuss the possibility of standardizing the features and how this relates
to estimating smoothness constants for the important problem of fitting
generalized linear models.

\subsection{Coordinate-wise Gradient Clipping}
\label{sub:clipping}

To bound the sensitivity of coordinate-wise gradients, our analysis of
Section~\ref{sec:diff-priv-coord} relies on the knowledge of Lipschitz
constants for the loss function $\ell(\cdot;d)$ that must hold for all
possible data points
$d\in \cX$, see
inequality \eqref{eq:delta-lipschitz-norm} and the discussion above
it.
This is classic in the analysis of DP optimization algorithms \citep[see
  e.g.,][]
{bassily2014Private,wang2017Differentially}. In practice however, these
Lipschitz constants can be difficult to bound tightly and often
give largely pessimistic estimates of sensitivities, thereby making gradients
overly noisy. To overcome this problem, the common practice in concrete
deployments of DP-SGD algorithms is to \emph{clip per-sample gradients}
so that their norm does not exceed a fixed threshold parameter $C > 0$
\citep{abadi2016Deep}:
\begin{align}
  \label{eq:standard_clipping_rule}
  \clip(\nabla \ell(w), C)
  = \min\Big(1, \frac{C}{\norm{\nabla \ell(w)}}_2\Big)  \nabla \ell(w)\enspace.
\end{align}
This effectively ensures that the sensitivity $\Delta(\clip(\nabla \ell, C))$
of the clipped gradient is bounded by $2C$.

In DP-CD, gradients are released one coordinate at a time and should
thus be clipped in a coordinate-wise fashion. Using the same threshold for
each coordinate would ruin the ability of DP-CD to account for imbalance
across gradient coordinates, whereas tuning coordinate-wise thresholds as $p$
individual
hyperparameters $\{C_j\}_{j=1}^p$
is impractical.

Instead, we leverage the results of \Cref{thm:cd-utility} to adapt them from a
single hyperparameter.
We first remark that our utility guarantees are
invariant to the scale of the matrix $M$. After rescaling
$M$ to
$\widetilde M = \frac{p}{\tr(M)} M$ so that
$\tr(\widetilde M) = \tr(I) = p$, as proposed by
\citet{richtarik2014Iteration}, the key quantity
$\Delta_{\widetilde M^{-1}}(\nabla\ell)$ as defined in \eqref{eq:delta-lipschitz-norm} appears in
our utility
bounds instead of $\norm{L}_{M^{-1}}$. This suggests to parameterize the
$j$-th threshold as $C_j = \sqrt{{M_j}/{\tr(M)}} C$ for some $C > 0$,
ensuring that $\Delta_{\widetilde M^{-1}}(\{\clip(\nabla_j\ell, C_j)\}_{j=1}^p)
  \leq 2C$.
The parameter $C$ thus controls the overall sensitivity, allowing clipped
DP-CD to perform $p$ iterations for the same privacy budget as one iteration
of clipped DP-SGD.

\subsection{Private Smoothness Constants}
\label{sec:priv-smoothn}

DP-CD requires the knowledge of the coordinate-wise smoothness
constants $M_1,\dots,M_p$ of $f$ to set appropriate step sizes (see
\Cref{thm:cd-utility}) and clipping thresholds (see
above).\footnote{In fact, only $\smash{M_j/\sum_{j'} M_{j'}}$ is needed, as we
  tune the clipping threshold and scaling factor for the step sizes.
  See \Cref{sec:numerical-experiments}.}  In most problems, the
$M_j$'s depend on the dataset $D$ and must thus be estimated privately
using a fraction of the overall privacy budget.  Since $f$ is an
average of loss terms, its coordinate-wise smoothness constants are
the average of those of $\ell(\cdot, d)$ over $d\in D$.  These per-sample
quantities are easy to get for typical losses (see
\Cref{sec:data-standardization} for the case of linear models).
Privately estimating $M_1,\dots,M_p$ thus reduces to a classic private
mean estimation problem for which many methods exist.  For instance,
assuming that the practitioner knows a crude upper bound on per-sample
smoothness constants, he/she can compute the smoothness constants of
the $\ell(\cdot, d)$'s, clip them to the pre-defined upper bounds, and
privately estimate their mean using the Laplace mechanism
(see \Cref{sec:priv-estim-smoothn} for
details).  We show numerically in
\Cref{sec:numerical-experiments} that dedicating $10\%$ of the total
budget $\epsilon$ to this strategy allows DP-CD to effectively exploit
the imbalance across gradients' coordinates.

\subsection{Feature Standardization}
\label{sec:data-standardization}

CD algorithms are very popular to solve generalized linear
models \citep{friedman2010Regularization} and their regularized version (\eg
LASSO, logistic regression).
For these problems, the coordinate-wise smoothness constants are
$M_j \propto \frac 1n \norm{X_{:,j}}_2^2$, where $X_{:,j} \in \RR^{n}$
is the vector containing the value of the $j$-th feature. Therefore, standardizing the features to have zero mean and
unit variance (a standard preprocessing step) makes coordinate-wise smoothness
constants equal. However, this requires to compute the
mean and variance of each feature in $D$, which is more
costly than the smoothness
constants to estimate privately.\footnote{We note that the privacy cost of
standardization is rarely accounted for in practical evaluations.}
Moreover, while our theory suggests that DP-CD may not be superior to DP-SGD
when smoothness constants are all equal (see
Section~\ref{sec:comparison-with-dp-sgd}), the
numerical results of \Cref{sec:numerical-experiments} show that DP-CD often
outperforms DP-SGD even when features are standardized.

Finally, we emphasize that standardization is not always possible.
This notably happens when solving the problem at hand is a subroutine of another algorithm.
For instance, the Iteratively Reweighted Least Squares (IRLS) algorithm
\citep{holland1977Robust} finds the maximum likelihood estimate of a
generalized linear model by solving a sequence of linear
regression problems with reweighted features, proscribing standardization.
Similar situations happen when using reweighted $\ell_1$ methods for
non-convex sparse regression \citep{Candes_Wakin_Boyd08}, relying on convex (LASSO) solvers for the inner loop.
DP-CD is thus a method of choice to serve as subroutine in private versions of these algorithms.

\begin{figure}[t]
  \captionsetup[subfigure]{justification=centering}
  \centering
  \begin{subfigure}{0.048\linewidth}
    \centering
    \includegraphics[width=\linewidth]{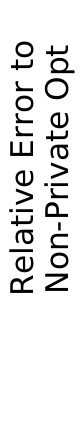}
    \begin{minipage}{.1cm}
      \vfill
    \end{minipage}
  \end{subfigure}%
  \begin{subfigure}{0.29\linewidth}
    \centering
    \includegraphics[width=\linewidth]{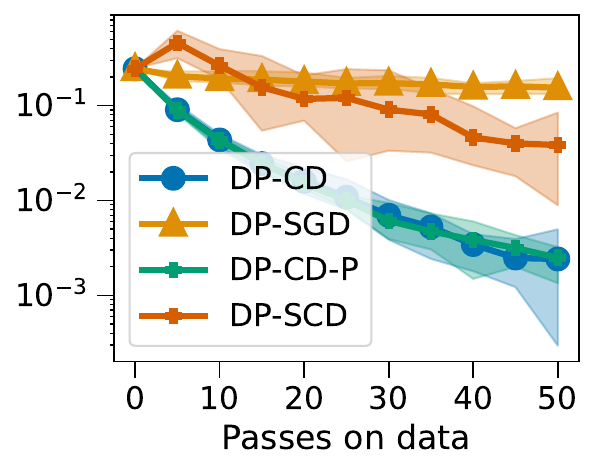}
    \caption{Electricity\\ Logistic, $\lambda=1/n$.}
    \label{fig:expe-logreg-constant}
  \end{subfigure}%
  \begin{subfigure}{0.29\linewidth}
    \centering
    \includegraphics[width=\linewidth]{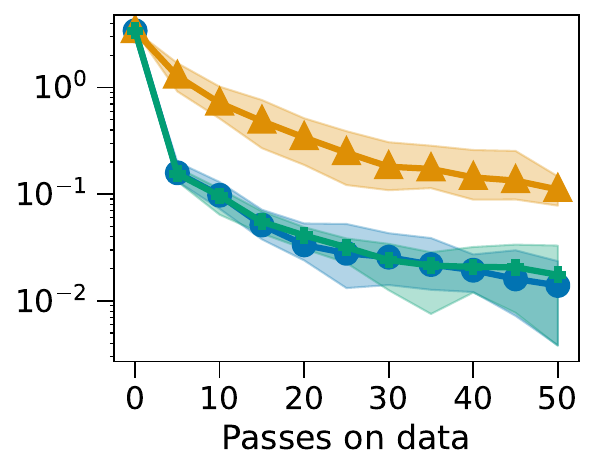}
    \caption{California\\ LASSO, $\lambda=1.0$.}
    \label{fig:expe-logreg-power-law-sparse}
  \end{subfigure}

  \caption{ Relative error to non-private optimal for DP-CD (blue), DP-CD with
  privately estimated coordinate-wise smoothness
    constants (green), DP-SGD (orange) and DP-SCD (red, only applicable to the
    smooth case) on
    two \emph{imbalanced} problems. The number of passes is
    tuned separately for each algorithm to achieve lowest error. We report
    min/mean/max values over 10~runs.}
  \label{fig:expe-raw}
\end{figure}

\section{Numerical Experiments}
\label{sec:numerical-experiments}

In this section, we assess the practical performance of DP-CD against
(proximal) DP-SGD on LASSO\footnote{\ie
  $\ell(w, (x,y)) = (w^\top x - y)^2$, $\psi(w) = \lambda\norm{w}_1$.}
and $\ell_2$-regularized logistic regression\footnote{\ie
  $\ell(w, (x, y)) = \log(1 + \exp(-y w^\top x))$,
  $\psi(w) = \!\frac{\lambda}{2}\!\norm{w}_2^2$.}. On the latter
problem, we also consider the dual private coordinate
descent algorithm of \citet{damaskinos2021Differentially}
(DP-SCD). For LASSO, we use the California dataset
\citep{kelleypace1997Sparse}, with $n=20,640$ records and $p=8$
features as well as a synthetic dataset (coined ``Sparse LASSO'') with
$n=1,000$ records and $p=1,000$ independent features that follow a
standard normal distribution.  The labels are then computed as a noisy
sparse linear combination of a subset of $10$ active features.  For
logistic regression, we consider the Electricity dataset
\citep{Electricity} with $45,312$ records and $8$ features.  On
California and Electricity, we set $\epsilon=1$ and $\delta=1/n^2$,
which is generally seen as a rather high privacy regime. The Sparse
LASSO dataset corresponds to a challenging setting for privacy
($n=p$), so we consider a low privacy regime with $\epsilon=10$,
$\delta=1/n^2$.  Privacy accounting for DP-SGD is done by numerically
evaluating the Rényi DP formula given by the sampled Gaussian
mechanism \citep{mironov2019Enyi}. Similarly for DP-CD, we do not use
the closed-form formula of Theorem~\ref{thm:dp-cd-privacy} but rather
numerically evaluate the tighter Rényi DP formula given in
Appendix~\ref{sec:proof-privacy}.

For DP-SGD, we use constant step sizes and standard gradient
clipping. For DP-CD, we adapt the coordinate-wise clipping thresholds
from one hyperparameter, as described in
\Cref{sub:clipping}. Similarly, coordinate-wise step sizes are set to
$\gamma_j = \gamma / M_j$, where $\gamma$ is a hyperparameter. When
the coordinate-wise smoothness constants are not all equal, we also
consider DP-CD with privately computed $M_j$'s, as described in
\Cref{sec:priv-smoothn}.  For each dataset and each algorithm, we
simultaneously tune the clipping threshold, the number of passes over
the dataset and, for DP-CD and DP-SGD, the step sizes.  After tuning
these parameters, we report the relative error to the (non-private)
optimal objective value.  The complete tuning procedure is described
in \Cref{sec:hyperp-tuning}, where we also give the best error for
various numbers of passes for each algorithm and dataset.  The code
used to obtain all our results is available in a public repository\footnote{
\url{https://gitlab.inria.fr/pmangold1/private-coordinate-descent/}}
and in the supplementary material.

\subsection{Imbalanced Datasets}
\label{sec:raw-datasets}

In the Electricity and California datasets, features are naturally
imbalanced. DP-CD can exploit this through the use of coordinate-wise
smoothness constants. We also consider a variant of DP-CD (DP-CD-P)
which dedicates $10\%$ of the privacy budget~$\epsilon$ to estimate
these constants (see \Cref{sec:priv-smoothn}) from a crude upper bound
on each feature (twice their maximal absolute value). It then uses the
resulting private smoothness constants in step sizes and clipping
thresholds. \Cref{fig:expe-raw} shows that DP-CD outperforms DP-SGD
and DP-SCD by an order of magnitude on both datasets, even when the
smoothness constants are estimated privately.

\subsection{Balanced Datasets}
\label{sec:stand-datas}

To assess the performance of DP-CD when coordinate-wise smoothness
constants are balanced, we standardize the Electricity and California
datasets (see Section~\ref{sec:data-standardization}).  As
standardization is done for all algorithms, we do not account for it
in the privacy budget. On standardized datasets, coordinate-wise
smoothness constants are all equal, removing the need of estimating
them privately. We report the results in \Cref{fig:expe-standardized}.
Although our theory suggests that DP-CD may do worse than DP-SGD in
balanced regimes, we observe that it still improves over DP-SGD (and
DP-SCD) in practice. Similar observations hold in our challenging
Sparse LASSO problem, where DP-SGD is barely able to make any
progress. We believe these results are in part due to the beneficial
effect of clipping in DP-CD, and the fact that DP-SGD relies on
amplification by subsampling, for which privacy accounting is not
perfectly tight. Additionally, CD methods are known to perform well on
fitting linear models: our results show that this transfers well to
private optimization.

\begin{figure*}[t]
  \captionsetup[subfigure]{justification=centering}
  \centering
  \begin{subfigure}{0.048\linewidth}
    \centering
    \includegraphics[width=\linewidth]{plots/xlegend.pdf}
    \begin{minipage}{.1cm}
      \vfill
    \end{minipage}
  \end{subfigure}%
  \begin{subfigure}{0.29\linewidth}
    \centering
    \includegraphics[width=\linewidth]{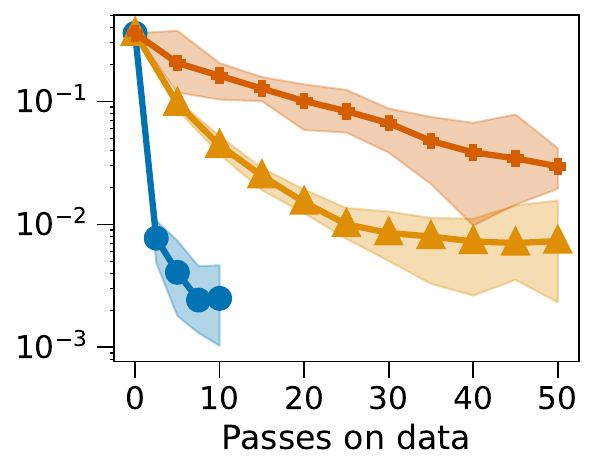}
    \caption{Electricity\\ Logistic, $\lambda=1/n$.}
    \label{subfig:expe-logreg-constant}
  \end{subfigure}%
  \begin{subfigure}{0.29\linewidth}
    \centering
    \includegraphics[width=\linewidth]{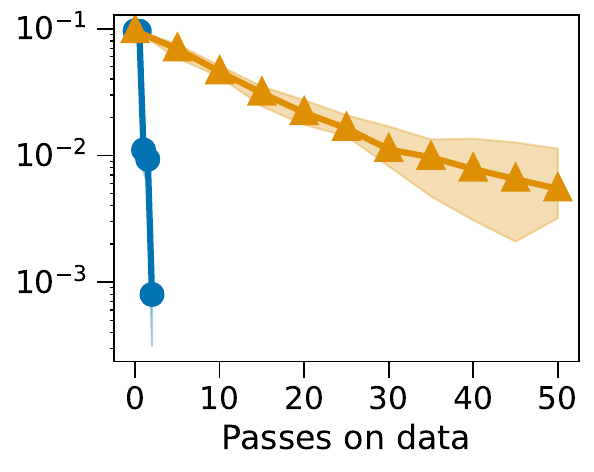}
    \caption{California\\ LASSO, $\lambda=0.1$.}
    \label{subfig:expe-logreg-power-law}
  \end{subfigure}
  \begin{subfigure}{0.29\linewidth}
    \centering
    \includegraphics[width=\linewidth]{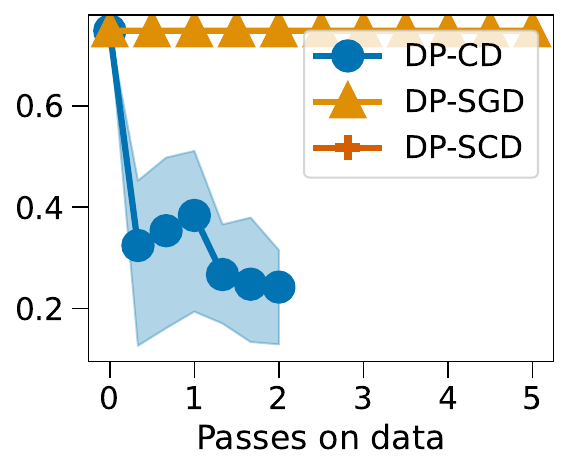}
    \caption{Sparse\\ LASSO, $\lambda=30$.}
    \label{fig:expe-logreg-power-law}
  \end{subfigure}

  \caption{
    Relative error to non-private optimal for  DP-CD (blue), DP-SGD (orange)
    and DP-SCD (red, only applicable to the smooth case) on three
    \emph{balanced} problems. The number of passes is tuned separately for
    each algorithm to achieve lowest error. We report min/mean/max values over
    10~runs.
  }
  \label{fig:expe-standardized}
\end{figure*}

\subsection{Running Time}
\label{sec:running-time}

The results above showed that DP-CD yields better
utility than DP-SGD. We also observe that DP-CD tends to reach
these
results in up to $10$~times fewer passes on the data than DP-SGD (see
\Cref{sec:hyperp-tuning} for detailed results). Additionally, when accounting
for running time, DP-CD
significantly outperforms DP-SGD: we refer to \Cref{sec:running-time-in} for
the counterparts of \Cref{fig:expe-raw} and \ref{fig:expe-standardized} as a
function of the running time instead of the number of passes.

\section{Related Work}
\label{sec:related-works}

\paragraph{DP-ERM.}
Differentially Private Empirical Risk Minimization was first studied by
\citet{chaudhuri2011Differentially}, using output perturbation (adding noise
to the solution of the non-private ERM problem) and objective perturbation
(adding noise to the ERM objective itself).
\citet{bassily2014Private} then proposed DP-SGD and proved its
near-optimality. %
\citet{wang2017Differentially} obtained faster convergence rates using a
DP version of the SVRG algorithm
\citep{johnson2013Accelerating,xiao2014Proximal}.
DP-SGD has
become the standard approach to DP-ERM.
In our work, we show that coordinate-wise updates can have lower sensitivity
than DP-SGD updates and propose a DP-CD algorithm achieving competitive
results.
A private variant of the Frank-Wolfe algorithm (DP-FW) was also
proposed to solve \emph{constrained} DP-ERM problems
\citep{talwar2015Nearly}.  Although these algorithms achieve a good
privacy-utility trade-off in theory, we are not aware of any empirical
evaluation.
DP-FW algorithms access gradients indirectly through a linear
optimization oracle over a constrained set.
Restricting to a constrained set is not necessary in DP-CD, allowing its use for a different family of problems.

\paragraph{DP-SCO.} Recent work has also studied algorithms and
utility guarantees for
stochastic convex optimization under differential privacy constraints, a
problem very similar to DP-ERM. \citet{bassily2019Private} \citep[following work
from][]{hardt2016Train,bassily2020Stability} extended results known
for DP-ERM to this setting, showing that the population risk of DP-SCO
is asymptotically equivalent to the one of non-private SCO. Efficient
algorithms for
DP-SCO
were proposed by \citet{feldman2020Private,wang2022Differentially},
and \citet{asi2021Private,bassily2021NonEuclidean} studied stochastic
variants of DP-FW. As detailed by
\citet{dwork2015Preserving,bassily2016Algorithmic,jung2021New} results
from DP-ERM can be converted to DP-SCO.

\paragraph{Coordinate descent.}
Coordinate descent (CD) algorithms have a long history in optimization.
\citet{Luo_Tseng1992,Tseng01,Tseng_Yun09} have shown convergence results for
(block) CD algorithms for nonsmooth optimization.
\citet{Nesterov12} later proved a global non-asymptotic $1/k$ convergence
rate for CD with random choice of coordinates for a convex, smooth objective.
Parallel, proximal variants were developed by
\citet{richtarik2014Iteration,fercoq2014Accelerated}, while
\citet{hanzely2018SEGA} further considered non-separable non-smooth parts.
\citet{shalev-shwartz2013Stochastic} introduced Dual CD algorithms
for smooth ERM, showing performance similar to SVRG.
We refer to \citet{wright2015Coordinate} and \citet{shi2017Primer} for
detailed reviews on CD.
Inexact CD was studied by \citet{tappenden2016Inexact}, but their analysis
requires updates not to increase the objective, which is hardly compatible
with DP.
We obtain tighter results for inexact CD with noisy gradients
(see Remark~\ref{rmq:improvement-inexact-coordinate-descent}).

\paragraph{Private coordinate descent.}
\citet{damaskinos2021Differentially} introduced a CD method to privately solve
the dual problem associated with generalized linear models with $\ell_2$
regularization. Dual CD is tightly related to SGD, as each
coordinate in the dual is associated with one data point.
The authors briefly mention the possibility of performing primal coordinate
descent but discard it on account of the seemingly large sensitivity of its
updates.
We show that primal DP-CD is in fact quite effective, and can be used to solve more general problems than considered by \citet{damaskinos2021Differentially}.
Primal CD was successfully used by \citet{bellet2018Personalized} to privately learn personalized models from decentralized datasets.
For the smooth objective they consider,
each coordinate depends only on a subset of the full dataset, which directly
yields low coordinate-wise sensitivity updates.
In contrast, we introduce a general algorithm for composite DP-ERM, for which
a novel utility analysis was required.

\section{Conclusion and Discussion}
\label{sec:conclusion-and-discussion}

We presented the first differentially private proximal coordinate descent
algorithm for composite DP-ERM.
Using an original approach to analyze proximal CD with perturbed
gradients, we derived optimal upper bounds on the privacy-utility trade-off
achieved by DP-CD. We also prove new lower bounds under a
component-Lipschitzness
assumption,
and showed that DP-CD matches these bounds.
Our results demonstrate that DP-CD strongly outperforms DP-SGD when
gradients' coordinates are imbalanced. Numerical experiments show that DP-CD
also performs very well in balanced regimes.
The choice of coordinate-wise clipping
thresholds is crucial for
DP-CD to achieve good utility in practice, and we provided a simple rule to
set them.

Although DP-CD already achieves good utility when most coordinates have small
sensitivity, our lower bounds suggest that even better utility could be
achieved by dynamically allocating more privacy budget to
coordinates with largest sensitivities.
A promising direction is to design DP-CD algorithms that leverage active set
methods
\citep{yuan2010Comparison,lewis2016Proximal,nutini2017Let,desantis2016Fast,Massias_Gramfort_Salmon18}, which could provide practical alternatives to recent DP-SGD approaches that use a
subspace assumption \citep{zhou2021Bypassing,kairouz2021Nearly}.
Finally, we believe that adaptive clipping techniques
\citep{pichapati2019AdaCliP,thakkar2019Differentially} may help to further
improve the practical performance of DP-CD when coordinate-wise
smoothness constants are more balanced.

\section*{Acknowledgments}

The authors would like to thank the anonymous reviewers who provided
useful feedback on previous versions of this work, which helped to improve
the paper.

This work was supported in part by the Inria Exploratory Action FLAMED and
by the French National Research Agency (ANR) through grant ANR-20-CE23-0015
(Project PRIDE) and ANR-20-CHIA-0001-01 (Chaire IA CaMeLOt).

\printbibliography

\appendix %
\newpage

\onecolumn

\section{Lemmas on Sensitivity}
\label{sec:lemma-sensitivity}

In this section, we let $\cX$ be the universe where the data is drawn from.
To upper bound the sensitivities of a function's gradient, we start by
recalling in \Cref{lemma:gradient-upper-bound} that (coordinate) gradients are
bounded by (coordinate-wise-)Lipschitz constants.
We then link this upper bound with gradients' sensitivities in
\Cref{lemma:gradient-sensitivity-upper-bound}.

\begin{lemma}
  \label{lemma:gradient-upper-bound}
  Let $\ell : \RR^p \times \cX \rightarrow \RR$ be convex and differentiable
  in its first argument,
  $\Lambda > 0$ and $L_1, \dots, L_p > 0$.
  \begin{enumerate}
  \item If $\ell(\cdot; d)$ is $\Lambda$-Lipschitz for all $d \in \cX$, then
    $\norm{\nabla \ell(w; d)}_2 \le \Lambda$ for all $w \in \RR^p$ and $d \in \cX$.
  \item If $\ell(\cdot; d)$ is $L$-component-Lipschitz for all $d \in \cX$, then
    $\abs{\nabla_j \ell(w; d)} \le L_j$ for all $w \in \RR^p$, $d \in \cX$
    and $j \in [p]$.
  \end{enumerate}
\end{lemma}

\begin{proof}
  Let $d \in \cX$.
  We start by proving the first statement.
  First, if $\nabla \ell(w;d) = 0$, $\norm{\nabla \ell(w;d)}_2 = 0 \le \Lambda$ and
  the
  result holds.
  Second, we focus on the case where $\nabla \ell(w;d) \not= 0$.
  The convexity of $\ell$ gives, for $w \in \RR^p$, $d \in \cX$:
  \begin{align}
    \ell(w + \nabla \ell(w;d);d)
    \ge \ell(w;d) + \scalar{\nabla \ell(w;d)}{\nabla \ell(w;d)}
    = \ell(w;d) + \norm{\nabla \ell(w;d)}_2^2\enspace,
  \end{align}
  then, reorganizing the terms and using $\Lambda$-Lipschitzness of $\ell$
  yields
  \begin{align}
    \norm{\nabla \ell(w;d)}_2^2
    \le \ell(w + \nabla \ell(w;d);d) - \ell(w;d)
    \le \abs{\ell(w + \nabla \ell(w;d);d) - \ell(w;d)}
    \le \Lambda \norm{\nabla \ell(w;d)}_2\enspace,
  \end{align}
  and the result follows after dividing by $\norm{\nabla \ell(w;d)}_2$.
  To prove the second statement, we set $j \in [p]$, and $w \in \RR^p$,
  and remark that if $\nabla_j \ell(w;d)=0$, then $\abs{\nabla_j \ell(w;d)} \le
  L_j$.
  When $\nabla_j \ell(w;d) \neq 0$, the convexity of $\ell$ yields
  \begin{align}
    \ell(w + \nabla_j \ell(w;d) e_j;d)
    \ge \ell(w;d) + \scalar{\nabla \ell(w;d)}{\nabla_j \ell(w;d) e_j}
    = \ell(w;d) + \nabla_j \ell(w;d)^2\enspace.
  \end{align}
  Reorganizing the terms and using $L$-component-Lipschitzness of $\ell$ gives
  \begin{align}
    \nabla_j \ell(w;d)^2
    \le \ell(w + \nabla_j \ell(w;d) e_j;d) - \ell(w;d)
    \le \abs{\ell(w + \nabla_j \ell(w;d) e_j;d) - \ell(w;d)}
    \le L_j \abs{\nabla_j\ell(w;d)}\enspace,
  \end{align}
  and we get the result after dividing by $\abs{\nabla_j \ell(w;d)}$.
\end{proof}

\begin{lemma}
  \label{lemma:gradient-sensitivity-upper-bound}
  Let $\ell : \RR^p \times \cX \rightarrow \RR$ be convex and differentiable
  in its 1st argument,
  $\Lambda > 0$ and $L_1, \dots, L_p > 0$.
  \begin{enumerate}
  \item If $\ell(\cdot; d)$ is $\Lambda$-Lipschitz for all $d \in \cX$, then
    $\Delta(\nabla \ell) \le 2\Lambda$.
  \item If $\ell(\cdot; d)$ is $L$-component-Lipschitz for all $d \in \cX$, then
    $\Delta(\nabla_j \ell) \le L_j$ for all $j \in [p]$.
  \end{enumerate}
\end{lemma}

\begin{proof}
  We start by proving the first statement.
  Let $w, w' \in \RR^p$, $d, d' \in \cX$.
  From the triangle inequality and \Cref{lemma:gradient-upper-bound}, we get
  the following
  upper bounds:
  \begin{align}
    \norm{\nabla \ell(w; d) - \nabla \ell(w'; d')}_2
    \le \abs{\nabla \ell(w; d)} + \abs{\nabla \ell(w'; d')}
    \le 2\Lambda\enspace,
  \end{align}
  which is the claim of the first statement.
  To prove the second statement, we proceed similarly: the triangle inequality
  and \Cref{lemma:gradient-upper-bound} give the following upper bounds:
  \begin{align}
    \abs{\nabla_j \ell(w; d) - \nabla_j \ell(w'; d')}
    \le \abs{\nabla_j \ell(w; d)} + \abs{\nabla_j \ell(w'; d')}
    \le 2L_j\enspace,
  \end{align}
  which is the desired result.
\end{proof}

We obtain the inequality \eqref{eq:delta-lipschitz-norm} stated in
Section~\ref{sec:preliminaries} as a corollary.

\begin{corollary}
  Let $L_1, \dots, L_p > 0$.
  Let $\ell(\cdot; d) : \RR^p \rightarrow \RR$ be a convex,
  $L$-component-Lipschitz function for all $d \in \cX$.
  Then
  \begin{align}
    \Delta_{M^{-1}}(\nabla \ell)
    = \Big(\sum_{j=1}^p \frac{1}{M_j} \Delta (\nabla_j\ell)^2\Big)^{\frac{1}
    {2}}
    \le \Big(\sum_{j=1}^p \frac{4}{M_j} L_j^2\Big)^{\frac{1}{2}}
    = 2 \norm{L}_{M^{-1}}\enspace.
  \end{align}
\end{corollary}

\section{Proof of \Cref{thm:dp-cd-privacy}}
\label{sec:proof-privacy}

To track the privacy loss of an adaptive composition of $K$ Gaussian
mechanisms, we use Rényi Differential Privacy
\citep[RDP]{mironov2017Renyi}. We note that similar results are
obtained with zero Concentrated Differential Privacy
\citep{bun2016Concentrated}. This flavor of differential privacy,
gives tighter privacy guarantees in that setting, as it reduces the
noise variance by a multiplicative factor of $\log(K/\delta)$ in
comparison to the usual advanced composition theorem of differential
privacy \citep{dwork2006Calibrating}.  Importantly, RDP can be
translated back to differential privacy.

In this section, we recall the definition and main properties of zCDP.
We denote by $\cD$ the set of all datasets over a universe $\cX$
and by $\cF$ the set of possible outcomes of the randomized
algorithms we consider.

\subsection{Rényi Differential Privacy}

We will use the Rényi divergence (\Cref{def:renyi-div}), which gives
a distribution-oriented vision of privacy.

\begin{definition}[Rényi divergence, \citealt{vanerven2014Renyi}]
  \label{def:renyi-div}
  For two random variables $Y$ and $Z$ with values in the same domain $\cC$,
  the Rényi divergence is, for $\alpha > 1$,
  \begin{align}
    D_\alpha(Y||Z)
    = \frac{1}{\alpha - 1} \log \int_{\cC} \prob{Y=z}^\alpha \prob{Z=z}^{1 - \alpha} dz\enspace.
  \end{align}
\end{definition}

We now define RDP in \Cref{def:rdp}. RDP provides a strong privacy
guarantee that can be converted to classical differential privacy
(\Cref{lemma:rdp-to-dp} and \Cref{cor:rdp-to-dp}).

\begin{definition}[Rényi Differential Privacy,
\citealt{mironov2017Renyi}]
  \label{def:rdp}
  A randomized algorithm $\cA : \mathcal D \rightarrow \mathcal F$ is
  $(\alpha, \epsilon)$-Rényi-differentially private (RDP) if, for all
  all datasets $D, D' \in \mathcal D$ differing on at most one
  element,
  \begin{align}
    D_\alpha(\cA(D) || \cA(D')) \le \epsilon\enspace.
  \end{align}
\end{definition}

\begin{lemma}[{\citealt[Proposition~3]{mironov2017Renyi}}]
  \label{lemma:rdp-to-dp}
  If a randomized algorithm $\cA : \mathcal D \rightarrow \mathcal F$
  is $(\alpha, \epsilon)$-RDP, then it is
  $(\epsilon + \frac{\log(1/\delta)}{\alpha - 1}, \delta)$-differentially
  private for all $0 < \delta < 1$.
\end{lemma}

\begin{remark}
  \label{rmq:rdp-to-dp}
  The above $(\alpha,\epsilon)$-RDP guarantees hold for multiple
  values of $\alpha,\epsilon$. As such, $\epsilon = \epsilon(\alpha)$
  can be seen as a function of $\alpha$, and \Cref{lemma:rdp-to-dp}
  ensures that the algorithm is $(\epsilon', \delta)$-DP for
  \begin{align}
    \epsilon' = \min_{\alpha > 1} \left\{ \epsilon(\alpha) + \frac{\log(1/\delta)}{\alpha - 1} \right\}\enspace.
  \end{align}
\end{remark}

We can now restate in \Cref{thm:rdp-composition} the composition
theorem of RDP, which is key in designing private iterative
algorithms.

\begin{theorem}[{\citealt[Proposition~1]{mironov2017Renyi}}]
  \label{thm:rdp-composition}
  Let $\cA_1, \dots, \cA_K : \cD \rightarrow \cF$ be $K > 0$
  randomized algorithms, such that for $1 \le k \le K$, $\cA_k$
  is $(\alpha, \epsilon_k(\alpha))$-RDP,
  where these algorithms can be chosen adaptively (i.e., $\cA_k$ can use
  to the output of $\cA_{k'}$ for all $k' < k$).
  Let $\cA : \mathcal D \rightarrow \mathcal F^K$ such that for $D \in \cD$,
  $\cA(D) = (\cA_1(D), \dots, \cA_K(D))$.
  Then $\cA$ is $\left(\alpha, \sum_{k=1}^K \epsilon_k(\alpha)\right)$-RDP.
\end{theorem}

Finally, we define the Gaussian mechanism (\Cref{def:gaussian-mechanism}), as used
in \Cref{algo:dp-cd}, and restate in \Cref{lemma:gaussian-rdp-private} the
privacy guarantees that it satisfies in terms of RDP.

\begin{definition}[Gaussian mechanism]
  \label{def:gaussian-mechanism}
  Let $f: \mathcal D \rightarrow \RR^p$, $\sigma > 0$, and $D \in \cD$.
  The Gaussian mechanism for answering the query $f$ is defined as:
  \begin{align}
    \cM_f^{Gauss}(D; \sigma) = f(D) + \mathcal N\left( 0, \sigma^2 I_p
    \right)\enspace.
  \end{align}
\end{definition}

\begin{lemma}[{\citealt[Corollary 3]{mironov2017Renyi}}]
  \label{lemma:gaussian-rdp-private}
  The Gaussian mechanism with noise $\sigma^2$ is
  $(\alpha, \frac{\Delta(f)^2 \alpha}{2 \sigma^2})$-RDP, where
  $\Delta(f) = \sup_{D,D'} \norm{f(D) - f(D')}_2$ (for neighboring
  $D, D'$) is the sensitivity of $f$.
\end{lemma}
\begin{proof}
  The function $h = \frac{f}{\Delta(f)}$ has sensitivity $1$, thus for
  any $s > 0$, the Gaussian mechanism $\cM_h^{Gauss}(\cdot;s)$ is
  $(\alpha, \frac{\alpha}{2\sigma^2})$-RDP
  \citep[Corollary~1]{mironov2017Renyi}. As $f = \Delta(f) \times h$,
  we have
  $\cM_f^{Gauss}(\cdot;\sigma) = \Delta(f) \times
  \cM_h^{Gauss}(\cdot;\frac{\sigma}{\Delta(f)})$.
  This mechanism is thus
  $(\alpha, \frac{\Delta(f)^2\alpha}{2\sigma^2})$-RDP.
\end{proof}

\begin{corollary}
  \label{cor:rdp-to-dp}
  Let $0 < \epsilon \le 1, 0 < \delta < \tfrac{1}{3}$. If a randomized
  algorithm $\cA : \mathcal D \rightarrow \mathcal F$ is
  $(\alpha, \frac{\gamma \alpha}{2 \sigma^2})$-RDP with $\gamma > 0$ and
  $\sigma = \frac{\sqrt{3 \gamma \log(1/\delta)}}{\epsilon}$ for all
  $\alpha > 1$, it is also $(\epsilon, \delta)$-DP.
\end{corollary}

\begin{proof}
  From \Cref{rmq:rdp-to-dp} it holds that $\cA$ is
  $(\epsilon',\delta)$-DP with
  $\epsilon' = \min_{\alpha > 1} \left\{ \frac{\gamma \alpha}{2 \sigma^2} +
    \frac{\log(1/\delta)}{\alpha - 1} \right\}.$ This minimum is
  attained when the derivative of the objective is zero, which is the
  case when
  $\frac{\gamma}{2\sigma^2} = \frac{\log(1/\delta)}{(\alpha - 1)^2}$,
  resulting in $\alpha = 1 + \sqrt{\frac{2\log(1/\delta)\sigma^2}{\gamma}}$.
  $\cA$ is thus $(\epsilon', \delta)$-DP with
  \begin{align}
  \label{eq:full_privacy_formula}
    \epsilon' = \frac{\gamma}{2\sigma^2} + \frac{\sqrt{\gamma \log(1/\delta)}}{\sqrt{2} \sigma}
    + \frac{\sqrt{\gamma \log(1/\delta)}}{\sqrt{2} \sigma}
     = \frac{\gamma}{2\sigma^2} + \frac{\sqrt{2\gamma \log(1/\delta)}}{\sigma}\enspace.
  \end{align}

  Choosing $\sigma = \frac{\sqrt{3 \gamma \log(1/\delta)}}{\epsilon}$
  now gives
  \begin{align}
    \epsilon'
    = \frac{\epsilon^2}{6 \log(1/\delta)} + \sqrt{{2}/{3}} \epsilon
    \le (1/6  + \sqrt{2/3}) \epsilon
    \le \epsilon\enspace,
  \end{align}
  where the first inequality comes from $\epsilon \leq 1$, thus $\epsilon^2
  \le \epsilon$
  and $\delta < 1/3$ thus $\frac{1}{\log(1/\delta)} \le 1$.
  The second inequality follows from $1/6  + \sqrt{2/3} \approx 0.983 < 1$.
\end{proof}

\subsection{Proof of \Cref{thm:dp-cd-privacy}}

We are now ready to prove \Cref{thm:dp-cd-privacy}.
From the privacy perspective, \Cref{algo:dp-cd} adaptively releases
and post-processes a series of gradient coordinates protected by the Gaussian
mechanism. We thus start by proving
\Cref{lemma:rdp-composition-gaussian-k}, which gives an $
(\epsilon,\delta)$-differential privacy
guarantee for the adaptive composition of $K$ Gaussian mechanisms.

\begin{lemma}
  \label{lemma:rdp-composition-gaussian-k}
  Let $0 < \epsilon \le 1$, $\delta < 1/3$, $K > 0$, $p > 0$, and
  $\{f_k : \RR^p \rightarrow \RR\}_{k=1}^{k=K}$ a family of $K$ functions.
  The adaptive composition of $K$ Gaussian mechanisms,
  with the $k$-th mechanism releasing
  $f_k$ with noise scale $\sigma_k = \frac{\Delta(f_k) \sqrt{3K \log(1/\delta)}}{\epsilon}$
  is $(\epsilon, \delta)$-differentially private.
\end{lemma}
\begin{proof}
  Let $\sigma > 0$. \Cref{lemma:gaussian-rdp-private} guarantees that
  the $k$-th Gaussian mechanism with noise scale
  $\sigma_k = \Delta(f_k) \sigma > 0$ is
  $(\alpha, \frac{\alpha}{2 \sigma^2})$-RDP.  Then, the composition of
  these $K$ mechanisms is, according to \Cref{thm:rdp-composition},
  $(\alpha, \frac{k\alpha}{2 \sigma^2})$-RDP.  This can be converted
  to $(\epsilon,\delta)$-DP via \Cref{cor:rdp-to-dp} with
  $\gamma = K$, which gives
  $\sigma_k = \frac{\Delta(f_k)\sqrt{3 k \log(1/\delta)}}{\epsilon}$
  for $k\in[K]$.
  \end{proof}

We now restate \Cref{thm:dp-cd-privacy} and prove it.

\begin{restate-theorem}{\ref{thm:dp-cd-privacy}}
  Assume $\ell(\cdot;d)$ is $L$-component-Lipschitz $\forall d\in\cX$.
  Let $\epsilon < 1$ and $\delta < 1/3$.
  If $\sigma_j^2 = \frac{12L_j^2 TK \log(1/\delta)}{n^2\epsilon^2}$
  for all $j \in [p]$,
  then \Cref{algo:dp-cd} satisfies $(\epsilon, \delta)$-DP.
\end{restate-theorem}

\begin{proof}
  For $j \in [1, p]$, $\nabla_j f$ in \Cref{algo:dp-cd} is released using the
  Gaussian mechanism with noise variance $\sigma_j^2$.
  The sensitivity of $\nabla_j f$
  is $\Delta(\nabla_j f) = \frac{\Delta (\nabla_j \ell)}{n} \le \frac{2L_j}
  {n}$. Note that $TK$ gradients are released, and
  \begin{align*}
    \sigma_j^2 = \frac{12L_j^2 TK \log(1/\delta)}{n^2\epsilon^2} \text{ for } j \in [1, p]\enspace,
  \end{align*}
  thus by \Cref{lemma:rdp-composition-gaussian-k} and the post-processing
  property of DP,
  \Cref{algo:dp-cd}
  is  $(\epsilon, \delta)$-differentially private.
\end{proof}

\section{Proof of Utility (\Cref{thm:cd-utility})}
\label{sec-app:proof-utility}

\subsection{Problem Statement}
\label{sec:problem-statement}

Let $D \in \cX^n$ be a dataset of $n$ elements drawn from a universe $\cX$.
Recall that we consider the following composite empirical risk minimization
problem:
\begin{align}
  \label{eq:comp_erm_supp}
  w^* \in
  \argmin_{w \in \mathbb{R}^p}
  \Bigg\{
  F(w; D) =
  \underbrace{\frac{1}{n} \sum_{i=1}^n \ell(w; d_i)}_{=: f(w; D)} + \psi(w)
  \Bigg\}\enspace,
\end{align}
where $\ell(\cdot, d)$ is convex, $L$-component-Lipschitz, and $M$-component-smooth
for all $d \in \cX$, and $\psi(w) = \sum_{j=1}^p \psi_j(w_j)$ is convex and
separable.
We denote by $F$ the complete objective function, and by $f$ its smooth part.
For readability, we omit the dependence on their second argument (\ie the
data) in the rest of this section.

\subsection{Proof of \Cref{thm:cd-utility}}
\label{sec:proof-thm-utility}

In this section, we prove our central theorem that guarantees the utility of
the DP-CD algorithm.
To this end, we start by proving a lemma that upper bounds the expected value
of $F(\theta^{k+1})$ in \Cref{algo:dp-cd}.
Using this lemma, we prove sub-linear convergence for the inner loop of DP-CD.
This gives the sub-linear convergence of our algorithm for convex losses.
Under the additional hypothesis that $F$ is strongly convex, we show that
iterates of the outer loop of DP-CD converge linearly towards the (unique)
minimum of
$F$.

We recall that in \Cref{algo:dp-cd}, iterates of the inner loop are denoted by
$\theta_1, \dots, \theta_K$, and those of the outer loop by $\bar w_1, \dots,
  \bar w_T$, with $\bar w_t = \frac{1}{K} \sum_{k=1}^K \theta^k$ for $t > 0$.
\Cref{algo:dp-cd} is randomized in two ways: when choosing the coordinate to
update and when drawing noise.
For convenience, we denote by $\expec{j}{\cdot}$ the expectation \wrt the
choice of coordinate,
by $\expec{\eta}{\cdot}$ the one \wrt the noise, and by $\expec{j,\eta}{\cdot}$
the expectation \wrt both.
When no subscript is used, the expectation is taken over all random
variables.
We will also use the notation $\condexpec{j,\eta}{\cdot}{\theta_k}$ for the
conditional expectation of a random variable, given a realization of
$\theta_k$.

\subsubsection{Descent Lemma}

We begin by proving \Cref{lemma:expec-first-upper-bound}, which decomposes the
change of a function $F$ when updating its argument $\theta \in \RR^p$,
in relation to a vector $w \in \RR^p$, into two parts:
one that remains fixed, corresponding to the unchanged entries of $\theta$,
and a second part corresponding to the objective decrease due to the
update.
At this point, the vector $w$ is arbitrary, but we will later choose $w$ to be
a minimizer of $F$, that is a solution to \eqref{eq:comp_erm_supp}.
\begin{lemma}
  \label{lemma:expec-first-upper-bound}
  Let $\ell, f, \psi,$ and $F$ be defined as in
  Section~\ref{sec:problem-statement}.
  Take a random variable $\theta \in \RR^p$ and two arbitrary vectors
  $w, g \in \RR^p$.
  Let a random variable $j$, taking its values uniformly randomly in $[p]$,
  Choose $\gamma_1, \dots, \gamma_p > 0$ and $\Gamma = \diag(\gamma_1, \dots, \gamma_p)$.
  It holds that%
  \begin{align}
    \label{eq:lemma-1-eq}
     & \condexpec{j}{F(\theta - \gamma_j g_j e_j) - F(w)}{\theta} - \frac{p - 1}{p} (F(\theta) - F(w)) \nonumber \\
     & \qquad \le \frac{1}{p} \left(f(\theta) - f(w) + \scalar{\nabla f(\theta)}{- \Gamma g}
    + \frac{1}{2}\norm{\Gamma g}_{M}^2 + \psi(\theta - \Gamma g) - \psi
    (w)\right)\enspace.
  \end{align}
\end{lemma}
\begin{remark}
  \label{rmq:expec-first-upper-bound-without-non-smooth}
  To avoid notational clutter, we will write $\gamma_j g_j$ instead of
  $\gamma_j g_j e_j$ throughout this section.
\end{remark}
\begin{proof}
  We start the proof by finding an upper bound on $\condexpec{j}{F(\theta - \gamma_j g_j e_j) - F(w)}{\theta}$,
  using the \textit{$M$-component-smoothness} of $f$:
  \begin{align}
     & \condexpec{j}{F(\theta - \gamma_j g_j e_j) - F(w)}{\theta}
    = \sum_{j=1}^p \frac{1}{p} (F(\theta - \gamma_j g_j) - F(w))                                     \\
    \overset{F=f+\psi}
     & {=} \frac{1}{p} \sum_{j=1}^p f(\theta - \gamma_j g_j) - f(w)
    + \psi(\theta - \gamma_j g_j) - \psi(w)                                                          \\
    \overset{f \text{ smooth}}%
     & {\le}
    \frac{1}{p} \sum_{j=1}^p \left( f(\theta) + \scalar{\nabla f(\theta)}{- \gamma_j g_j}
    + \frac{1}{2}\norm{\gamma_j g_j}_M^2 - f(w)
    + \psi(\theta - \gamma_j g_j) - \psi(w) \right)                                                  \\
     & = f(\theta) - f(w) + \frac{1}{p} \sum_{j=1}^p \left(\scalar{\nabla f(\theta)}{- \gamma_j g_j}
    + \frac{1}{2}\norm{\gamma_j g_j}_{M}^2
    + (\psi(\theta - \gamma_j g_j) - \psi(w)) \right)                                                \\
     & = f(\theta) - f(w) + \frac{1}{p} \scalar{\nabla f(\theta)}{- \Gamma g}
    + \frac{1}{2p}\norm{\Gamma g}_{M}^2
    + \frac{1}{p}\sum_{j=1}^p(\psi(\theta - \gamma_j g_j) - \psi(w))\enspace. \label{eq:upper-bound-expectation-smoothness}
  \end{align}

  The regularization terms can now be reorganized using the separability of $\psi$, as done by \cite{richtarik2014Iteration}.
  Indeed, we notice that
  \begin{align}
    \sum_{j=1}^p \left(\psi(\theta - \gamma_j g_j) - \psi(w)\right)
     & = \sum_{j=1}^p \Big(\psi_j(\theta_j - \gamma_j g_j) - \psi_j(w_j)
    + \sum_{j'\neq j} \psi_{j'}(\theta_{j'}) - \psi(w_{j'})\Big)                                                          \\
     & = \psi(\theta - \Gamma g) - \psi(w) + (p - 1) (\psi(\theta) - \psi(w))\enspace. \label{eq:reorganize-terms-of-regularizer}
  \end{align}

  Plugging \eqref{eq:reorganize-terms-of-regularizer} in
  \eqref{eq:upper-bound-expectation-smoothness} results in the following:
  \begin{align}
    \condexpec{j}{F(\theta - \gamma_j g_j e_j) - F(w)}{\theta}
     & \le f(\theta) - f(w) + \frac{1}{p} \scalar{\nabla f(\theta)}{- \Gamma g}
    + \frac{1}{2p}\norm{\Gamma g}_{M}^2 \nonumber                                   \\
     & \qquad + \frac{1}{p} (\psi(\theta - \Gamma g) - \psi(w))
    + \frac{p - 1}{p} (\psi(\theta) - \psi(w))                                      \\
     & = \frac{1}{p} \left(f(\theta) - f(w) + \scalar{\nabla f(\theta)}{- \Gamma g}
    + \frac{1}{2}\norm{\Gamma g}_{M}^2
    + \psi(\theta - \Gamma g) - \psi(w)
    \right) \nonumber                                                               \\
     & \qquad + \frac{p - 1}{p} (f(\theta) + \psi(\theta) - f(w) - \psi(w))\enspace,         %
  \end{align}
  which gives the lemma since $F = f + \psi$.
\end{proof}

To exploit this result, we need to upper bound the right hand side of
\eqref{eq:lemma-1-eq} for the realizations of $\theta^k$ in \Cref{algo:dp-cd}.
This is where our proof differs from classical convergence proofs for coordinate
descent methods.
Namely, we rewrite the right hand side of
\eqref{eq:lemma-1-eq} so as to obtain telescopic terms plus a
bias term resulting from the addition of noise, as shown in \Cref{lemma:descent-lemma}.

\begin{lemma}
  \label{lemma:descent-lemma}
  Let $\ell, f, \psi,$ and $F$ defined as in Section~\ref{sec:problem-statement}.
  For $k > 0$, let $\theta^k$ and $\theta^{k+1}$ be two consecutive iterates
  of the
  inner loop of \Cref{algo:dp-cd}, $\gamma_1 = \frac{1}{M_1}, \dots, \gamma_p = \frac{1}{M_p} > 0$
  the coordinate-wise step sizes
  (where $M_j$ are the coordinate-wise smoothness constants of $f$),
  and $g_j = \frac{1}{\gamma_j} (\theta^{k+1}_j - \theta^k_j)$.
  Let $w \in \RR^p$ an arbitrary vector and $\sigma_1, \dots, \sigma_p > 0$ the
  coordinate-wise noise scales given as input to \Cref{algo:dp-cd}.
  It holds that
  \begin{align}
    \label{eq:descent-lemma-eq}
    & \condexpec{j, \eta}{F(\theta^{k+1}) - F(w)}{\theta^k}
    - \frac{p - 1}{p} (F(\theta^k) - F(w)) \nonumber \\
    & \qquad \le \tfrac{1}{2} \norm{\theta^k - w}_{\Gamma^{-1}}^2 - \tfrac{1}{2}\condexpec{j,\eta}{\norm{\theta^{k+1} - w}_{\Gamma^{-1}}^2}{\theta^k}
    + \tfrac{1}{p}\norm{\sigma}_\Gamma^2\enspace,
  \end{align}
  where $\norm{\sigma}_\Gamma^2 = \sum_{j=1}^p \gamma_j \sigma_j^2$ and the
  expectations are taken over the random choice of $j$ and $\eta$,
  conditioned upon the realization of $\theta^k$.
\end{lemma}

\begin{proof}
  We define $g$ the vector $(g_1, \dots, g_p) \in \RR^p$ with
  $g_j = \frac{1}{\gamma_j} (\theta^{k+1}_j - \theta^k_j)$ when coordinate $j$
  is chosen in \Cref{algo:dp-cd}.
  We also denote by $\Gamma = \diag(\gamma_1, \dots, \gamma_p)$ the diagonal
  matrix having the step sizes as its coefficients.

  From \Cref{lemma:expec-first-upper-bound} with $\theta = \theta^k$, $w = w$ and
  $g = g$ as defined above we obtain
  \begin{align}
    \label{eq:lemma-1-eq-applied-for-theta-k}
     & \condexpec{j}{F(\theta^k - \gamma_j g_j e_j) - F(w)}{\theta^k} - \frac{p - 1}{p} (F(\theta^k) - F(w)) \nonumber \\
     & \qquad \le \frac{1}{p} \left(f(\theta^k) - f(w) + \scalar{\nabla f(\theta^k)}{- \Gamma g}
    + \frac{1}{2}\norm{\Gamma g}_{M}^2 + \psi(\theta^k - \Gamma g) - \psi
    (w)\right)\enspace.
  \end{align}
  We can upper bound the right hand term of~\eqref{eq:lemma-1-eq-applied-for-theta-k}
  using the convexity of $f$ and $\psi$:
  \begin{align}
     & f(\theta^k) - f(w) + \scalar{\nabla f(\theta^k)}{- \Gamma g}
    + \frac{1}{2}\norm{\Gamma g}_{M}^2
    + \psi(\theta^k - \Gamma g) - \psi(w)                                                                \\
     & \qquad \le \scalar{\nabla f(\theta^k)}{\theta^k - w} + \scalar{\nabla f(\theta^k)}{- \Gamma g}
    + \frac{1}{2}\norm{\Gamma g}_{M}^2
    + \scalar{\partial \psi(\theta^k - \Gamma g)}{\theta^k - \Gamma g - w}                               \\
     & \qquad = \scalar{\nabla f(\theta^k) + \partial \psi(\theta^k - \Gamma g)}{\theta^k- \Gamma g - w}
    + \frac{1}{2}\norm{\Gamma g}_{M}^2\enspace,
  \end{align}
  where we use the slight abuse of notation $\partial \psi(\theta^k - \Gamma g)$ to
  denote any vector in the subdifferential of $\psi$ at the point $\theta^k -
    \Gamma g$. %
  We now rewrite the dot product:
  \begin{align}
     & \scalar{\nabla f(\theta^k) + \partial \psi(\theta^k - \Gamma g)}{\theta^k- \Gamma g - w}
    + \frac{1}{2}\norm{\Gamma g}_{M}^2                                                             \\
     & \qquad = \scalar{g}{\theta^k- \Gamma g - w}
    + \frac{1}{2}\norm{\Gamma g}_{M}^2
    + \scalar{\nabla f(\theta^k) + \partial \psi(\theta^k - \Gamma g) - g}{\theta^k- \Gamma g - w} \\
     & \qquad = \underbrace{\scalar{g}{\theta^k- w}
    - \norm{g}^2_{\Gamma}
    + \frac{1}{2}\norm{g}_{\Gamma^2M}^2}_{\text{``descent'' term}}
    + \underbrace{\scalar{\nabla f(\theta^k)
        + \partial \psi(\theta^k - \Gamma g) - g}{\theta^k- \Gamma g - w}}_{\text{``noise'' term}}\enspace,
    \label{lemma:descent:decomp-two-terms}
  \end{align}
  where the second equality follows from $\scalar{g}{-\Gamma g} = -\norm{g}_{\Gamma}^2$
  and $\norm{\Gamma g}_M^2 = \norm{g}_{\Gamma^2M}^2$.
  We split \eqref{lemma:descent:decomp-two-terms} into two terms: a ``descent'' term and a ``noise'' term.

  \textbf{Rewriting the ``descent'' term.}
  We first focus on the ``descent'' term.
  As $\gamma_j = \frac{1}{M_j}$ for all $j \in [p]$, it holds that
  $\gamma_j^2 M_j = \gamma_j$ which gives $- \norm{g}^2_{\Gamma} + \frac{1}{2}\norm{g}_{\Gamma^2M}^2
    = - \norm{g}^2_{\Gamma} + \frac{1}{2}\norm{g}_{\Gamma}^2
    = - \frac{1}{2} \norm{g}^2_{\Gamma}$.
  We can now rewrite the ``descent'' term as a difference of
  two norms, materializing the distance to $w$, weighted by the inverse of the
  step sizes $\Gamma^{-1}$:
  \begin{align}
    \text{``descent'' term}
     & = \scalar{g}{\theta^k- w}
    - \frac{1}{2} \norm{g}^2_{\Gamma}                   \\
     & =
    \scalar{\Gamma g}{\theta^k- w}_{\Gamma^{-1}}
    - \frac{1}{2} \norm{\Gamma g}^2_{\Gamma^{-1}}       \\
     & =
    \frac{1}{2}\norm{\theta^k - w}_{\Gamma^{-1}}^2
    - \frac{1}{2}\norm{\theta^k - w}_{\Gamma^{-1}}^2
    + \scalar{\Gamma g}{\theta^k- w}_{\Gamma^{-1}}
    - \frac{1}{2} \norm{\Gamma g}^2_{\Gamma^{-1}}       \\
     & = \frac{1}{2}\norm{\theta^k - w}_{\Gamma^{-1}}^2
    - \frac{1}{2}\norm{\theta^k - \Gamma g - w}_{\Gamma^{-1}}^2\enspace,
    \label{lemma:descent:descent-term}
  \end{align}
  where we factorized the norm to obtain the last inequality.
  We can rewrite \eqref{lemma:descent:descent-term} as an expectation over the
  random choice of the coordinate $j$ (drawn uniformly in $[p]$), given the
  realizations of $\theta^k$ and of the noise $\eta$ (which determines $g$):
  \begin{align}
    \frac{1}{2} \norm{\theta^k - w}_{\Gamma^{-1}}^2 - \frac{1}{2} \norm{\theta^k - \Gamma g - w}_{\Gamma^{-1}}^2
     & = \frac{p}{2} \times \left( \frac{1}{p} \sum_{j=1}^p \gamma_j^{-1} \abs{\theta^k_j - w_j}^2 - \gamma_j^{-1} \abs{\theta_j^k - \gamma_j g_j - w_j}^2 \right) \\
     & = \frac{p}{2} \times \condexpec{j}{\gamma_j^{-1} \abs{\theta^k_j - w_j}^2 - \gamma_j^{-1} \abs{\theta_j^k - \gamma_j g_j - w_j}^2}{\theta^k, \eta}\enspace.
  \end{align}
  Finally, we remark that
  $\gamma_j^{-1} \abs{\theta^k_j - w_j}^2 - \gamma_j^{-1} \abs{\theta_j^k - \gamma_j g_j - w_j}^2
    = \norm{\theta^k - w}_{\Gamma^{-1}}^2 - \norm{\theta^{k} - \gamma_j g_j - w}_{\Gamma^{-1}}^2$,
  as only one coordinate changes between the two vectors, and the squared norm
  $\norm{\cdot}_{\Gamma^{-1}}^2$ is separable.
  We thus obtain
  \begin{align}
    \text{``descent'' term}
     & = \condexpec{j}{ \frac{p}{2} \norm{\theta^k - w}_{\Gamma^{-1}}^2 - \frac{p}{2} \norm{\theta^{k} - \gamma_j g_j - w}_{\Gamma^{-1}}^2}{\theta^k, \eta} \\
     & = \frac{p}{2} \norm{\theta^k - w}_{\Gamma^{-1}}^2 - \frac{p}{2} \condexpec{j}{\norm{\theta^{k+1} - w}_{\Gamma^{-1}}^2}{\theta^k, \eta}\enspace.
  \end{align}
  \textbf{Upper bounding the ``noise'' term.}
  We now upper bound the ``noise'' term in~\eqref{lemma:descent:decomp-two-terms}.
  We first recall the definition of the noisy proximal update $g_j$ (line~\ref{algo-line:coordinate-minimization-update} of \Cref{algo:dp-cd}),
  and define its non-noisy counterpart $\tilde g_j$:
  \begin{align}
    g_j        & = \gamma_j^{-1} \Big(\prox_{\gamma_j \psi_j}(\theta^k_j - \gamma_j (\nabla_j f(\theta^k) + \eta_j)) - \theta^k_j\Big) \\
    \tilde g_j & = \gamma_j^{-1} \Big( \prox_{\gamma_j \psi_j}(\theta^k_j - \gamma_j (\nabla_j f(\theta^k)) - \theta^k_j\Big)\enspace.
  \end{align}
  For an update of the coordinate $j \in [p]$, the optimality condition of the
  proximal operator gives, for $\eta_j$ the realization of the noise drawn
  at the current iteration when coordinate $j$ is chosen:
  \begin{align}
    0
     & \in \theta^{k+1}_j - \theta^k_j + \gamma_j (\nabla_j f(\theta^k) + \eta_j)) + \frac{1}{M_j} \partial \psi_j(\theta_j^k - \gamma_j g_j) \\
     & = \gamma_j \times \left(
    \frac{1}{\gamma_j} (\theta^{k+1}_j - \theta^k_j) + \nabla_j f(\theta^k) + \eta_j + \partial \psi_j(\theta_j^k - \gamma_j g_j) \right)\enspace.
  \end{align}
  As such, there exists a real number
  $v_j \in \partial \psi_j(\theta^k_j - \gamma_j g_j)$ such that
  $g_j = - \frac{1}{\gamma_j} (\theta_j^{k+1} - \theta_j^k) = \nabla_j f(\theta^k) + \eta_j + v_j$.
  We denote by $v \in \RR^p$ the vector having this $v_j$ as $j$-th coordinate.
  Recall that $\psi$ is separable, therefore $v \in \partial \psi(\theta^k - \Gamma g)$.
  The ``noise'' term of \eqref{lemma:descent:decomp-two-terms} can be thus
  be rewritten using $v$:
  \begin{align}
    \text{``noise'' term}
    = \scalar{\nabla f(\theta^k) + v - g}{\theta^k- \Gamma g - w}
    = \scalar{\eta}{\theta^k- \Gamma g - w}\enspace,
  \end{align}
  and we now separate this term in two using $\widetilde g$:
  \begin{align}
    \text{``noise'' term}
    = \sum_{j=1}^p \eta_j (\theta^k_j - \gamma_j g_j - w_j)
    = \sum_{j=1}^p \eta_j (\theta^k_j - \gamma_j \widetilde g_j - w_j)
    + \sum_{j=1}^p \eta_j (\gamma_j \widetilde g_j - \gamma_j g_j)\enspace. \label{lemma:descent:noise-term-two-terms}
  \end{align}
  It is now time to consider the expectation with respect to the noise of
  these terms.
  First, as $\widetilde g_j$ is not dependent on the noise anymore, it simply
  holds that
  \begin{align}
    \bbE_{\eta}\Big[\sum_{j=1}^p \eta_j (\theta^k_j - \gamma_j \widetilde g_j - w_j) \mid \theta^k \Big]
    = \sum_{j=1}^p \expec{\eta}{\eta_j} (\theta^k_j - \gamma_j \widetilde g_j - w_j)
    = 0\enspace. \label{lemma:descent:noise-term-zero-upper-bound}
  \end{align}

  The last step of our proof now takes care of the following term:
  \begin{align}
    \bbE_\eta \Big[ \sum_{j=1}^p \eta_j (\gamma_j \widetilde g_j - \gamma_j g_j) \mid \theta^k \Big]
    \le \bbE_{\eta}\Big[ \gamma_j \Big|\sum_{j=1}^p \eta_j (\widetilde g_j - g_j)\Big| \mid \theta^k \Big]
    \le \sum_{j=1}^p \gamma_j \condexpec{\eta}{~\abs{\eta_j} \abs{\widetilde g_j - g_j}~}{~\theta^k}\enspace,
  \end{align}
  where each inequality comes from the triangle inequality.
  The non-expansiveness property of the proximal operator (see
  \citet{parikh2014Proximal},
  Section 2.3) is now key to our result, as it yields
  \begin{align}
    \abs{\widetilde g_j - g_j}
    = \gamma_j^{-1} \abs{\prox_{\gamma_j \psi_j}(\theta^k_j - \gamma_j (\nabla_j f(\theta^k)))
    - \prox_{\gamma_j \psi_j}(\theta^k_j - \gamma_j (\nabla_j f(\theta^k) + \eta_j)) }
    \le \abs{\eta_j}\enspace,
  \end{align}
  which directly gives, as $\expec{\eta}{\eta_j^2} = \sigma_j^2$
  (and $\norm{\sigma}_\Gamma^2 = \sum_{j=1}^p \gamma_j \sigma_j^2$),
  \begin{align}
    \sum_{j=1}^p \gamma_j \condexpec{\eta}{\abs{\eta_j} \abs{\widetilde g_j - g_j}}{\theta^k}
    \le \sum_{j=1}^p \gamma_j \expec{\eta}{\abs{\eta_j} \abs{\eta_j}}
    = \sum_{j=1}^p \gamma_j \expec{\eta}{\eta_j^2}
    = \norm{\sigma}_\Gamma^2\enspace. \label{lemma:descent:noise-term-upper-bound}
  \end{align}
  We now have everything to prove the lemma by plugging
  \eqref{lemma:descent:noise-term-upper-bound} and \eqref{lemma:descent:noise-term-zero-upper-bound}
  into expected value of \eqref{lemma:descent:noise-term-two-terms}, and then
  \eqref{lemma:descent:noise-term-two-terms} and \eqref{lemma:descent:descent-term}
  back into \eqref{lemma:descent:decomp-two-terms} to obtain, after using
  the Tower property of conditional expectations:
  \begin{align}
     & \frac{1}{p}\condexpec{j,\eta}{f(\theta^k) - f(w) + \scalar{\nabla f(\theta^k)}{- \Gamma g}
      + \frac{1}{2}\norm{\Gamma g}_{M}^2
      + \psi(\theta^k - \Gamma g) - \psi(w)}{\theta^k}                                                                                              \\
     & \qquad \le \frac{1}{p} \left( \text{``descent'' term} + \text{``noise'' term} \right)                                                        \\
     & \qquad \le \frac{1}{2}\norm{\theta^k - w}_{\Gamma^{-1}}^2 - \frac{1}{2}\condexpec{j,\eta}{\norm{\theta^{k+1} - w}_{\Gamma^{-1}}^2}{\theta^k}
    + \frac{1}{p}\norm{\sigma}_\Gamma^2\enspace,
  \end{align}
  which is the result of the lemma.
\end{proof}

\subsubsection{Convergence Lemma}

\Cref{lemma:descent-lemma} allows us to prove a result on the mean of $K$
consecutive noisy coordinate-wise gradient updates, by simply summing it and
rewriting the terms.
This gives \Cref{lemma:dp-cd-convergence-lemma}, which is the key lemma of our
proof.
\begin{lemma}
  \label{lemma:dp-cd-convergence-lemma}
  Assume $\ell(\cdot, d)$ is convex, $L$-component-Lipschitz and
  $M$-component-smooth for all $d \in \cX$, $\psi$ is convex and separable, such
  that $F = f + \psi$ and $w^*$ is a minimizer of $F$.
  For $t \in [T]$, consider the $K$ successive iterates $\theta^1, \dots, \theta^K$
  computed from the inner loop of \Cref{algo:dp-cd} starting from the point $\bar w^t$,
  with step sizes $\gamma_j = \frac{1}{M_j}$ and noise scales
  $\sigma_j$.
  Letting $\bar w^{t+1}= \frac{1}{K} \sum_{k=1}^K \theta^k$,
  it holds that
  \begin{align}
    \expec{}{F(\bar w^{t+1}) - F(w^*)}
     & \le \frac{p(\norm{\bar w^t - w^*}_M^2 + 2(F(\bar w^t) - F(w^*)))}{2K}
    + \norm{\sigma}_{M^{-1}}^2\enspace.
  \end{align}
\end{lemma}

\begin{remark}
  The term $F(\bar w^t) - F(w^*)$ essentially remains in the inequality due to
  the composite nature of $F$.
  When %
  $\psi = 0$,
  $M$-component-smoothness of $f(\cdot; d)$ (for $d \in \cX$) gives
  \begin{align}
    f(\bar w^t)
    \le f(w^*) + \scalar{\nabla f(w^*)}{\bar w^t - w^*} + \frac{1}{2}\norm{\bar w^t - w^*}_M^2
    = f(w^*) + \frac{1}{2}\norm{\bar w^t - w^*}_M^2\enspace,
  \end{align}
  and the result of \Cref{lemma:dp-cd-convergence-lemma} further
  simplifies as:
  \begin{align}
    \expec{}{F(\bar w^{t+1}) - F(w^*)}
     & \le \frac{p\norm{\bar w^t - w^*}_M^2}{K}
    + \norm{\sigma}_{M^{-1}}^2\enspace.
  \end{align}
\end{remark}

\begin{proof}
  Summing \Cref{lemma:descent-lemma} for $k = 0$ to $k = K$ and $w = w^*$,
  taking expectation with respect to all choices of coordinate and random noise
  and using the tower property gives:
  \begin{align}
     & \sum_{k=0}^{K-1} \expec{}{F(\theta^{k+1}) - F(w^*)}
    - \frac{p - 1}{p} \sum_{k=0}^{K-1} \expec{}{(F(\theta^k) - F(w^*))} \nonumber              \\
     & \qquad \le \sum_{k=0}^{K-1} \frac{1}{2} \expec{}{\norm{\theta^k - w^*}_{\Gamma^{-1}}^2}
    - \frac{1}{2} \expec{}{\norm{\theta^{k+1} - w^*}_{\Gamma^{-1}}^2}
    + \frac{1}{p}\norm{\sigma}_\Gamma^2                                                        \\
     & \qquad = \frac{1}{2} \expec{}{\norm{\bar w^0 - w^*}_{\Gamma^{-1}}^2}
    - \frac{1}{2} \expec{}{\norm{\theta^{K} - w^*}_{\Gamma^{-1}}^2}
    + \frac{K}{p}\norm{\sigma}_\Gamma^2\enspace. \label{lemma-eq:dp-cd-cv-convex:telescopic}
  \end{align}

  Remark that $\sum_{k=0}^{K-1} \expec{}{F(\theta^k) - F(w^*)}
    = \sum_{k=1}^{K} \expec{}{F(\theta^k) - F(w^*)}
    + (F(\bar w^0) - F(w^*))
    - \expec{}{F(\theta^K) - F(w^*)}$,
  then as $\expec{}{F(\theta^K) - F(w^*)} \ge 0$, we obtain a lower bound on the
  left hand side of~\eqref{lemma-eq:dp-cd-cv-convex:telescopic}:
  \begin{align}
    \sum_{k=0}^{K-1} \expec{}{F(\theta^{k+1}) - F(w^*)}
    - \tfrac{p - 1}{p} \sum_{k=0}^{K-1} \expec{}{(F(\theta^k) - F(w^*))}
     & \ge \tfrac{1}{p}\sum_{k=1}^{K} \expec{}{F(\theta^{k}) - F(w^*)}
    - (F(\bar w^0) - F(w^*))\enspace.
    \label{lemma-eq:dp-cd-cv-convex:telescopic-lhs}
  \end{align}

  As $\bar w^{t+1} = \frac{1}{K} \sum_{k=1}^K \theta^k$, the convexity of $F$
  gives $F(\bar w^{t+1}) \le \frac{1}{K} \sum_{k=1}^K  F(\theta^{k}) - F(w^*)$.
  Plugging this inequality into
  \eqref{lemma-eq:dp-cd-cv-convex:telescopic-lhs} and combining the result with
  \eqref{lemma-eq:dp-cd-cv-convex:telescopic} gives
  \begin{align}
    F(\bar w^{t+1}) - F(w^*)
     & \le \frac{p(\frac{1}{2}\norm{\bar w^0 - w^*}_{\Gamma^{-1}}^2 + F(\bar w^0) - F(w^*))}{K}
    + \norm{\sigma}_\Gamma^2\enspace.
  \end{align}
  We conclude the proof by using the fact that $\Gamma_j = M_j^{-1}$ for all
  $j \in [p]$,
  thus $\norm{\cdot}_\Gamma = \norm{\cdot}_{M^{-1}}$ and
  $\norm{\cdot}_{\Gamma^{-1}} = \norm{\cdot}_{M}$.
\end{proof}

\subsubsection{Convex Case}

\begin{restate-theorem}{\ref{thm:cd-utility}}[Convex case]
Let $w^*$ be a minimizer of $F$ and $R_M^2 = \max(\norm{\bar w^0 -
    w^*}^2_M, F(\bar w^0) - F(w^*))$.
The output $w^{priv}$ of DP-CD (Algorithm~\ref{algo:dp-cd}), starting from
$\bar w^0 \in \RR^p$ with $T = 1$, $K > 0$ and the $\sigma_j$'s as in
\Cref{thm:dp-cd-privacy}, satisfies:
\begin{align}
  F(w^{priv}) - F(w^*)
   & \le \frac{3p R_M^2}{2K}
  + \frac{12 \norm{L}_{M^{-1}}^2 K\log(1/\delta)}{n^2\epsilon^2}\enspace.
\end{align}

Setting $K = \frac{R_M \sqrt{p} n \epsilon}{\norm{L}_{M^{-1}}\sqrt{8\log(1/\delta)}}$
yields:
\begin{align}
  F(w^{priv}) - F(w^*)
   & \le \frac{9\sqrt{p} \norm{L}_{M^{-1}}R_M \sqrt{\log(1/\delta)}}{n\epsilon}
  = \widetilde O \left( \frac{\sqrt{p} R_M \norm{L}_{M^{-1}}}{n\epsilon} \right)\enspace.
\end{align}
\end{restate-theorem}

\begin{proof}
  In the convex case, we iterate only once in the inner loop (since $T=1$).
  As such, $w^{priv} = \bar w^{1}$, and applying \Cref {lemma:dp-cd-convergence-lemma}
  with $\bar w^{t+1} = \bar w^{1}$, $w^t = \bar w^0$ and $\sigma_j$ chosen as in
  \Cref{thm:dp-cd-privacy} gives the result.
  Taking $K = \frac{R_M \sqrt{p} n \epsilon}{\norm{L}_{M^{-1}}\sqrt{8\log(1/\delta)}}$
  then gives
  \begin{align}
    F(\bar w^{t+1}_1) - F(w^*)
     & \le \frac{2\sqrt{8p\log(1/\delta)} \norm{L}_{M^{-1}} R_M}{n\epsilon}
    + \frac{12\sqrt{p\log(1/\delta)}\norm{L}_{M^{-1}} R_M}{\sqrt{8}n\epsilon}\enspace,
  \end{align}
  and the result follows from $2\sqrt{8} + \frac{12}{\sqrt{8}} \approx 8.48 < 9$.
\end{proof}

\subsubsection{Strongly Convex Case}

\begin{restate-theorem}{\ref{thm:cd-utility}}[Strongly-convex case]
Let $F$ be $\mu_M$-strongly convex w.r.t. $\norm{\cdot}_M$ and $w^*$ be the
minimizer of $F$. The output
$w^{priv}$ of DP-CD (Algorithm~\ref{algo:dp-cd}), starting from
$\bar w^0 \in \RR^p$ with $T > 0$, $K = 2p ( 1 + 1 / \mu_M )$
and the $\sigma_j$'s as in \Cref{thm:dp-cd-privacy}, satisfies:
\begin{align}
  F(w^{priv}) - F(w^*)
   & \le \frac{F(\bar w^0) - F(w^*)}{2^T}
  + \frac{24 p (1 + 1/\mu_M) T \norm{L}_{M^{-1}}^2 \log(1/\delta)}{n^2\epsilon^2}\enspace.
\end{align}

Setting $T = \log_2 \left(\frac{32 n^2\epsilon^2 (F(\bar w^0) - F(w^*))}{p(1+1/\mu_M)\norm{L}_{M^{-1}}^2\log(1/\delta)}\right)$ yields:
\begin{align}
  \expec{}{F(w^{priv}) - F(w^*)}
   & \le \left(1 + \log_2\left(\frac{(F(\bar w^0) - F(w^*)) n^2 \epsilon^2}{24 p (1+1/\mu_M)\norm{L}_{M^{-1}}^2 \log(1/\delta)}\right)\right)
  \frac{24 p (1+1/\mu_M)\norm{L}_{M^{-1}}^2 \log(1/\delta)}{n^2 \epsilon^2}                                                                   \\
   & = O\left(
  \frac{p\norm{L}_{M^{-1}}^2 \log(1/\delta)}{\mu_M n^2 \epsilon^2}
  \log_2\left(\frac{(F(\bar w^0) - F(w^*)) n \epsilon \mu_M}{p\norm{L}_{M^{-1}}\log(1/\delta)}\right)
  \right)\enspace.
\end{align}
\end{restate-theorem}

\begin{proof}
  As $F$ is $\mu_M$-strongly-convex with respect to norm $\norm{\cdot}_M$,
  we obtain for any $w \in \RR^p$, that $F(w) \ge F(w^*) + \frac{\mu_M}{2} \norm{w - w^*}_M^2$.
  Therefore, $F(\bar w^0) - F(w^*) \le \frac{2}{\mu_M} \norm{\bar w^0 - w^*}_M^2$
  and \Cref {lemma:dp-cd-convergence-lemma} gives, for $1 \le t \le T-1$,
  \begin{align}
    F(\bar w^{t+1}) - F(w^*)
     & \le \frac{(1+1/\mu_M) p(F(\bar w^t) - F(w^*))}{K}
    + \norm{\sigma}_M^2\enspace.
  \end{align}

  It remains to set $K = 2 p (1 + 1/\mu_M)$ to obtain
  \begin{align}
    F(\bar w^{t+1}) - F(w^*)
     & \le \frac{F(\bar w^t) - F(w^*)}{2}
    + \norm{\sigma}_M^2\enspace.
  \end{align}

  Recursive application of this inequality gives
  \begin{align}
    \label{lemma:dp-cd-strong-convexity-convergence:eq:conv}
    \expec{}{F(\bar w^T) - F(w^*)}
     & \le \frac{F(\bar w^0) - F(w^*)}{2^T} + \sum_{t=0}^{T-1} \frac{1}{2^t} \norm{\sigma}_M^2
    \le \frac{F(\bar w^0) - F(w^*)}{2^T} + 2 \norm{\sigma}_M^2\enspace,
  \end{align}
  where we upper bound the sum by the value of the complete series.
  It remains to replace $\norm{\sigma}_M^2$ by its value to obtain the result.
  Taking
  $T = \log_2\left(\frac{(F(\bar w^0) - F(w^*)) n^2 \epsilon^2}{24 p (1+1/\mu_M)\norm{L}_{M^{-1}}^2 \log(1/\delta)}\right)$
  then gives
  \begin{align}
    \expec{}{F(\bar w^T) - F(w^*)}
     & \le \left(1 + \log_2\left(\frac{(F(\bar w^0) - F(w^*)) n^2 \epsilon^2}{24 p (1+1/\mu_M)\norm{L}_{M^{-1}}^2 \log(1/\delta)}\right)\right)
    \frac{24 p (1+1/\mu_M)\norm{L}_{M^{-1}}^2 \log(1/\delta)}{n^2 \epsilon^2}                                                                   \\
     & = O\left(
    \frac{p\norm{L}_{M^{-1}}^2 \log(1/\delta)}{\mu_M n^2 \epsilon^2}
    \log_2\left(\frac{(F(\bar w^0) - F(w^*)) n \epsilon \mu_M}{p\norm{L}_{M^{-1}}\log(1/\delta)}\right)
    \right)\enspace,
  \end{align}
  which is the result of our theorem.
\end{proof}

\subsection{Proof of Remark 1}
\label{sec:proof-remark-1}

We recall the notations of \citet{tappenden2016Inexact}.
For $\theta \in \RR^p$, $t \in \RR$ and $j \in [p]$, let
$V_j(\theta, t) = \nabla_j(\theta) t + \frac{M_j}{2} \abs{t}^2 + \psi_j(\theta^k_j + t)$.
For $\eta \in \RR$, we also define its noisy counterpart,
$V_j^{\eta}(\theta, t) = (\nabla_j(\theta) + \eta) t + \frac{M_j}{2} \abs{t}^2 + \psi_j(\theta^k_j + t)$.
We aim at finding $\delta_j$ such that for any $\theta^k \in \RR^p$ used in the
inner loop of \Cref{algo:dp-cd}:
\begin{align}
  \expec{\eta_j}{V_j(\theta^k, -\gamma_j g_j)}
  \le \min_{\widetilde g \in \RR} V_j(\theta^k, -\gamma_j \widetilde g) + \delta_j\enspace,
\end{align}
where the expectation is taken over the random noise $\eta_j$, and
$- \gamma_j g_j = \prox_{\gamma_j \psi_j}(\theta^k_j - \gamma_j (\nabla_j f(\theta^k) + \eta_j)) - \theta^k_j$
as defined in the analysis of \Cref{algo:dp-cd}.
We need to link the proximal operator we use in DP-CD with the quantity
$V_j^{\eta_j}$ that we just defined:
\begin{align}
  \prox_{\gamma_j \psi_j} ( \theta^k_j - \gamma_j (\nabla_j f(\theta^k) + \eta_j) )
   & = \argmin_{v \in \RR}  \frac{1}{2} \norm{v - \theta^k_j + \gamma_j (\nabla_j f(\theta^k) + \eta_j)}_2^2                                                 \\
   & = \argmin_{v \in \RR}  \scalar{\gamma_j (\nabla_j f(\theta^k_j) + \eta_j)}{v - \theta^k_j} + \frac{1}{2} \norm{v - \theta^k_j}_2^2 + \gamma_j \psi_j(v) \\
   & = \argmin_{v \in \RR}  \scalar{\nabla_j f(\theta^k) + \eta_j}{v - \theta^k_j} + \frac{M_j}{2} \norm{v - \theta^k_j}_2^2 + \psi_j(v)                     \\
   & = \theta^k_j + \argmin_{t \in \RR}  \scalar{\nabla_j f(\theta^k) + \eta_j}{t} + \frac{M_j}{2} \norm{t}_2^2 + \psi_j(\theta^k_j + t)\enspace.
\end{align}
Which means that
$-\gamma_j g_j
  =\prox_{\gamma_j \psi_j} ( \theta^k_j - \gamma_j (\nabla_j f(\theta^k) + \eta_j) ) -  \theta^k_j
  \in \argmin_{t\in\RR} V_j^{\eta_j}(\theta^k,t)$.
Let $-\gamma_j g_j^* = \prox_{\gamma_j \psi_j}(\theta^k_j - \gamma_j \nabla_j(\theta^k)) - \theta^k_j$
be the non-noisy counterpart of $-\gamma_j g_j$.
Since $-\gamma_j g_j$ is a minimizer of $V_j^{\eta_j}(\theta^k, \cdot)$, it holds that
\begin{align}
  V_j^{\eta_j}(\theta^k, -\gamma_j g_j)
   & \le \scalar{\nabla_j f(\theta^k) + \eta_j}{-\gamma_j g_j^*} + \frac{M_j}{2} \norm{-\gamma_j g_j^*}_2^2 + \psi_j(\theta^k_j + -\gamma_j g_j^*) \\
   & = \min_t V_j(\theta^k, t) + \scalar{\eta_j}{-\gamma_j g_j^*}\enspace,
\end{align}
which can be rewritten as $V_j(\theta^k, -\gamma_j g_j) \le \min_t V_j(\theta^k, t) + \scalar{\eta_j}{\gamma_j (g_j - g_j^*)}$.
Taking the expectation yields
\begin{align}
  \expec{\eta_j}{V_j(\theta^k, -\gamma_j g_j)} \le \min_t V_j(\theta^k, t) + \expec{\eta_j}{\scalar{\eta_j}{\gamma_j (g_j - g_j^*)}}\enspace.
\end{align}
Finally, we remark that $\abs{g_j - g_j^*} \le \abs{\gamma_j \eta_j}$ and the
non-expansiveness of the proximal operator gives
\begin{align}
  \expec{\eta_j}{V_j(\theta^k, -\gamma_j g_j)} \le \min_t V_j(\theta^k, t) + \gamma_j \sigma_j^2\enspace,
\end{align}
which implies an upper bound on the expectation of $\delta_j$:
$\expec{j,\eta_j}{\delta_j}
  = \frac{1}{p} \sum_{j=1}^p \expec{\eta_j}{\delta_j}
  \le \frac{1}{p} \sum_{j=1}^p \gamma_j \sigma_j^2 = \frac{1}{p} \sum_{j=1}^p \sigma_j^2 / M_j$,
when $\gamma_j = 1/M_j$.
In the formalism of \citet{tappenden2016Inexact},
this amounts to setting $\alpha = 0$ and $\beta = \frac{1}{p} \norm{\sigma}_{M^{-1}}^2$.

\paragraph{Convex functions.}

When the objective function $F$ is convex, we use \Cref{lemma:dp-cd-convergence-lemma}
to obtain, since $\norm{\sigma}_{M^{-1}}^2 = \beta p$,
\begin{align}
  F(w^{1}) - F(w^*) \le \frac{2pR_M^2}{K} + \norm{\sigma}_{M^{-1}}^2
  = \frac{2pR_M^2}{K} + \beta p\enspace.
\end{align}
Therefore, when $F$ is convex, we get $F(w^1) - F(w^*) \le \xi$, for $\xi > \beta p$,
as long as $\frac{2pR_M^2}{K} \le \xi - \beta p$, that is $K \ge \frac{2pR_M^2}{\xi - \beta p}$.

In comparison, \citet[Theorem 5.1 therein]{tappenden2016Inexact} gives
convergence
to $\xi > \sqrt{2pR_M^2 \beta}$ when $K \ge \frac{2pR_M^2}{\xi - \sqrt{2pR_M^2\beta} }$.
We thus gain a factor $\sqrt{\beta p / 2R_M^2}$ in utility.
Importantly, our utility upper bound does not depend on initialization in that
setting, whereas the one of \citet{tappenden2016Inexact} does.

\paragraph{Strongly-convex functions.}

When the objective function $F$ is $\mu_M$-strongly-convex \wrt to $\norm{\cdot}_M$, then
from~\eqref{lemma:dp-cd-strong-convexity-convergence:eq:conv} we obtain, as long
as $K \ge 4/\mu_M$, that
\begin{align}
  \expec{}{F(w^T) - F(w^*)} \le \frac{F(w^0) - F(w^*)}{2^T} + 2 \beta p\enspace.
\end{align}
This proves that $\expec{}{F(w^T) - F(w^*)} \le \xi$ for $\xi > 2\beta p$ when
$\frac{F(w^0) - F(w^*)}{2^T} \le \xi - 2 \beta p$ that is
$T \ge \log\frac{F(w^0) - F(w^*)}{\xi - 2\beta p}$ and
$TK \ge \frac{4p}{\mu_M}\log\frac{F(w^0) - F(w^*)}{\xi - 2\beta p}$.
In comparison, \citet[Theorem 5.2 therein]{tappenden2016Inexact} shows
convergence
to $\xi > \frac{\beta p}{\mu_M}$ for
$K \ge \frac{p}{\mu_M} \log \frac{F(w^0) - F(w^*) - \frac{\beta p}{\mu_M}}{\xi - \frac{\beta p}{\mu_M} }$.
We thus gain a factor $\mu_M/2$ in utility.

\section{Comparison with DP-SGD}
\label{sec-app:comparison-with-dp}

In this section, we provide more details on the arguments of
\Cref{sec:comparison-with-dp-sgd}, where we suppose that $\ell$ is $L$-component-Lipschitz
and $\Lambda$-Lipschitz.
To ease the comparison, we assume that $R_M = \norm{w^0 - w^*}_M$, which is
notably the case in the smooth setting with $\psi = 0$
(see Remark~\ref{rmq:expec-first-upper-bound-without-non-smooth}).

\paragraph{Balanced.}
We start by the scenario where coordinate-wise smoothness constants are balanced
and all equal to $M = M_1 = \cdots = M_p$.
We observe that
\begin{align}
  \norm{L}_{M^{-1}}
  = \sqrt{\sum_{j=1}^p \frac{1}{M_j} L_j^2}
  = \sqrt{\frac{1}{M} \sum_{j=1}^p L_j^2}
  = \frac{1}{\sqrt{M}} \norm{L}_2\enspace.
\end{align}
We then consider the convex and strongly-convex functions separately:
\begin{itemize}
  \item \textit{Convex functions:}
        it holds that $R_M = \sqrt{M} R_I$, which yields the equality
        $\norm{L}_{M^{-1}} R_{M} = \norm{L}_2 R_I$.
  \item \textit{Strongly convex functions:}
        if $f$ is $\mu_M$-strongly-convex with respect to $\norm{\cdot}_M$, then for
        any $x, y \in \RR^p$,
        \begin{align}
          f(y)
          \ge f(x) + \scalar{\nabla f(x)}{y - x} + \frac{\mu_M}{2} \norm{y - x}_M^2
          = f(x) + \scalar{\nabla f(x)}{y - x} + \frac{M\mu_M}{2} \norm{y - x}^2_2\enspace,
        \end{align}
        which means that $f$ is $M\mu_M$-strongly-convex with respect to $\norm{\cdot}_2$.
        This gives
        $\frac{\norm{L}_{M^{-1}}^2}{\mu_M}
          = \frac{\norm{L}_2^2 / M}{\mu_I / M}
          = \frac{\norm{L}_2^2}{\mu_I}$.
\end{itemize}

In light of the results summarized in \Cref{table:utility-cd-sgd}, it remains to
compare $\norm{L}_2 = \sqrt{\sum_{j=1}^p L_j^2}$ with $\Lambda$, for which it
holds that $\Lambda \le \sqrt{\sum_{j=1}^p L_j^2} \le \sqrt{p}\Lambda$, which is
our result.

\paragraph{Unbalanced.}
When smoothness constants are disparate, we discuss the case where
\begin{itemize}
  \item \textit{one coordinate of the gradient dominates the others:}
        we assume without loss of generality that the dominating coordinate is the first one.
        It holds that $M_1 =: M_{\max} \gg M_{\min} =: M_j$, for all $j \neq 1$ and
        $L_1 =: L_{\max} \gg L_{\min} =: L_j$, for all $j \neq 1$ such that
        $\frac{L_1^2}{M_1} \gg \sum_{j\neq 1} \frac{L_j^2}{M_j}$.
        As $L_1$ dominates the other component-Lipschitz constants, most of the variation of the loss comes from its first coordinate.
        This implies that $L_1$ is close to the global Lipschitz constant $\Lambda$ of $\ell$.
        As such, it holds that
        \begin{align}
          \norm{L}_{M^{-1}}^2
          =
          \sum_{j=1}^p \frac{L_j^2}{M_j}
          \approx \frac{L_1^2}{M_1}
          \approx \frac{\Lambda^2}{M_{\max}}\enspace.
        \end{align}
  \item \textit{the first coordinate of $\bar w^0$ is already very close to
          its optimal value}
        so that $M_1\abs{\bar w^0_1 - w^*_1} \ll \sum_{j \neq 1} M_j \abs{\bar w^0_j - w^*_j}$.
        Under this hypothesis,
        \begin{align}
          R_M^2
          \approx \sum_{j\neq 1}M_j\abs{w_j^0 - w_j^*}^2
          = M_{\min} \sum_{j\neq 1}\abs{w_j^0 - w_j^*}^2
          \approx M_{\min} R_I^2\enspace.
        \end{align}
\end{itemize}
We can now easily compare DP-CD with DP-SGD in this scenario.
First, if $\ell$ is convex, then
$\norm{L}_{M^{-1}} R_M \approx \sqrt{\frac{M_{\min}}{M_{\max}}} \Lambda R_I$.
Second, when $\ell$ is strongly-convex, we observe that for $x, y \in \RR^p$,
\begin{align}
  f(y)
  \ge f(x) + \scalar{\nabla f(x)}{y - x} + \frac{\mu_M}{2} \norm{y - x}_M^2
  \ge f(x) + \scalar{\nabla f(x)}{y - x} + \frac{M_{\min} \mu_M}{2} \norm{y - x}_2^2\enspace,
\end{align}
which implies that when $f$ is $\mu_M$ strongly-convex with respect to
$\norm{\cdot}_M$, it is $M_{\min}\mu_M$ strongly-convex with respect
to $\norm{\cdot}_2$.  This yields, under our hypotheses,
$\frac{\norm{L}_{M^{-1}}^2}{\mu_M} \approx \frac{\Lambda^2 /
  M_{\max}}{\mu_I / M_{\min}} = \frac{M_{\min}}{M_{\max}}
\frac{\Lambda^2}{\mu_I}$.
In both cases, DP-CD can get arbitrarily better than DP-SGD, and gets
better as the ratio ${M_{\max}}/{M_{\min}}$ increases.

The two hypotheses we describe above are of course very restrictive.
However, it gives some insight about when and why DP-CD can outperform
DP-SGD.
Our numerical experiments in \Cref{sec:numerical-experiments} confirm this
analysis, even in less favorable
cases.

\section{Proof of Lower Bounds}
\label{sec:utility-lower-bounds-1}

To prove lower bounds on the utility of $L$-component-Lipschitz functions, we
extend the proof of \citet{bassily2014Private} to our setting
(that is, $L$-component-Lipschitz functions and unconstrained
composite optimization).
There are three main difficulties in adapting their proof:
\begin{itemize}
  \item First, the optimization problem \eqref{eq:dp-erm} is not constrained.
        We stress that while convex constraints can be enforced using the
        regularizer $\psi$ (using the characteristic function of a convex set),
        its separable nature only allows box constraints. In contrast,
        \citet{bassily2014Private} rely on an $\ell_2$-norm constraint to obtain
        their lower bounds.
  \item Second, Lemma 5.1 of \citet{bassily2014Private} must be extended to
        our $L$-component-Lipschitz setting.
        To do so, we consider datasets with points in $\prod_{j=1}^p \{-L_j, L_j\}$
        rather than $\{-1/\sqrt{p}, 1/\sqrt{p}\}^p$, and carefully adapt the
        construction of
        the dataset $D$ so that $\norm{\sum_{i=1}^n d_i}_2 = \Omega(\min(n\norm{L}_2, {\sqrt{p} \norm{L}_2}/{\epsilon}))$,
        which is essential to prove our lower bounds.
  \item Third, the lower bounds of \citet{bassily2014Private} rely on fingerprinting
        codes, and in particular on the result of \citet{bun2014Fingerprinting} which
        uses such codes to prove that (when $n$ is smaller than
        some $n^*$ we describe later) differential privacy is incompatible with
        precisely and simultaneously estimating \emph{all} $p$ counting
        queries defined over the columns of the dataset $D$.
        In our construction, since all columns of $D$ now have different scales, we
        need an additional
        hypothesis on the repartition of the $L_j$'s, (\ie that
        $\sum_{j\in\cJ} L_j^2 = \Omega(\norm{L}_2)$ for all $\cJ \subseteq [p]$ of a
        given size), which is not
        required in existing lower bounds (where all columns have equal scale).
\end{itemize}

\subsection{Counting Queries and Accuracy}
\label{sec:accur-re-ident}

We start our proof by recalling and extending to our setting the notions of
counting queries (\Cref{def:counting-query}) and accuracy (\Cref{def:accuracy}),
as described by \citet{bun2014Fingerprinting}.
The main feature of our definitions is that we allow the set $\cX$ to have
different scales for each of its coordinates, and that we account for this scale
in the definition of accuracy.
We denote by $\conv(\cX)$ the convex hull of a set $\cX$.

\begin{definition}[Counting query]
  \label{def:counting-query}
  Let $n > 0$.
  A counting query on $\cX$ is a function $q : \cX^n \rightarrow \conv(\cX)$ defined
  using a predicate   $q : \cX \rightarrow \cX$.
  The evaluation of the query $q$ over a dataset $\cD \in \cX^n$ is defined as
  the arithmetic mean of $q$ on $\cD$:
  \begin{align}
    q(\cD) = \frac{1}{n} \sum_{i=1}^n q(d_i)\enspace.
  \end{align}
\end{definition}

\begin{definition}[Accuracy]
  \label{def:accuracy}
  Let $n, p \in \NN$, $\alpha, \beta \in [0, 1]$, $L_1, \dots, L_p > 0$,
  and $\cX = \prod_{j=1}^p \{-L_j; L_j\}$ or $\cX = \{0, L_j\}^p$.
  Let $\cQ = \{q_1, \dots, q_p\}$ be a set of $p$ counting queries on $\cX$ and
  $D \in \cX^n$ a dataset of $n$ elements.
  A sequence of answers $a = (a_1, \dots, a_p)$ is said $(\alpha, \beta)$-accurate
  for $\cQ$ if $\abs{q_j(D) - a_j} \le L_j\alpha$ for at least a
  $1 - \beta$ fraction of indices $j \in [p]$.
  A randomized algorithm $\cA : \cX^n \rightarrow \RR^{\card{\cQ}}$ is said
  $(\alpha, \beta)$-accurate for $\cQ$ on $\cX$ if for every $D \in \cX^n$,
  \begin{align}
    \prob{\cA(D) \text{ is } (\alpha, \beta)\text{-accurate for } \cQ} \ge 2/3\enspace.
  \end{align}
\end{definition}

In our proof, we will use a specific class of queries: one-way marginals
(\Cref{def:one-way-marginals}), that compute the arithmetic mean of a dataset
along one of its column.
\begin{definition}[One-way marginals]
  \label{def:one-way-marginals}
  Let $\cX = \prod_{j=1}^p \{-L_j; L_j\}$ or $\cX = \{0, L_j\}^p$.
  The family of one-way marginals on $\cX$ is defined by queries with predicates
  $q_j(x) = x_j$ for $x \in \cX$.
  For a dataset $D \in \cX^n$ of size $n$, we thus have
  $q_j(D) = \frac{1}{n} \sum_{i=1}^n d_{i,j}$.
\end{definition}

\subsection{Lower Bound for One-Way Marginals}
\label{sec:lower-bounds-1-1}

We can now restate a key result from \citet{bun2014Fingerprinting}, which shows
that there exists a minimal number $n^*$ of records needed in a dataset to allow
achieving both accuracy and privacy on the estimation of one-way marginals on
$\cX=(\{0, 1\}^p)^n$.
This lemma relies on the construction of re-identifiable distribution
(see \citealt[Definition~2.10]{bun2014Fingerprinting}). One can then use this
distribution to find a dataset on which a private algorithm can not be accurate
(see \citealt[Lemma~2.11]{bun2014Fingerprinting}).
\begin{lemma}[{\citealt[Corollary~3.6]{bun2014Fingerprinting}}]
  \label{lemma:bun-fingerprint-1way-marginal}
  For $\epsilon > 0$ and $p > 0$, there exists a number $n^* = \Omega(\frac{\sqrt{p}}{\epsilon})$
  such that for all $n \le n^*$, there exists no algorithm that is both
  $(1/3, 1/75)$-accurate and $(\epsilon, o\left(\frac{1}{n}\right))$-differentially
  private for the estimation of one-way marginals on $(\{0, 1\}^p)^n$.
\end{lemma}

To leverage this result in our setting of private empirical risk minimization,
we start by extending it to queries on $\cX = \prod_{j=1}^p \{-L_j; L_j\}$.
Before stating the main theorem of this section (\Cref{thm:lower-bound-one-way-marginals}),
we describe a procedure $\chi_L : (\{0, 1\}^p)^n \rightarrow \cX^{3n}$
(with $L_1, \dots, L_p > 0$), that takes as input a dataset $D \in (\{0,
  1\}^p)^n$ and outputs an augmented and rescaled version. This procedure
is crucial to our proof and is defined as follows.
First, it adds $2n$ rows filled with $1$'s to $D$, which ensures that the sum
of each column of $D$ is $\Theta(n)$ (which gives the lower bound on $M$ in
\Cref{thm:lower-bound-one-way-marginals}).
Then it rescales each of these columns by subtracting $1/2$ to each coefficient
and multiplying the $j$-th column of $D$ ($j \in [p]$) by $2L_j$.
The resulting dataset $D_L^{aug} = \chi_L(D)$ is a set of $3n$ points with values
in $\cX = \prod_{j=1}^p \{-L_j, L_j\}$, with the property that, for all
$j \in [p]$, $3nL_j \ge \sum_{i=1}^n (D_L^{aug})_{i,j} \ge nL_j$.
For $D \in (\{0, 1\}^p)^n$, we show how to reconstruct $q_j(\chi_L(D))$
from $q_j(D)$
in \Cref{claim:qj-qjaug}.

\begin{claim}
  \label{claim:qj-qjaug}
  Let $n \in \NN$, $j \in [p]$, $L_j > 0$ and $q_j$ the $j$-th one-way marginal
  on datasets $D$ with $p$ columns such that for $d_i \in D$, $q_j(d_i) = d_{i,j}$.
  Let $D_L^{aug} = \chi_L(D)$.
  It holds that
  \begin{align}
    q_j(D_L^{aug}) = \frac{2L_j}{3} q_j(D) + \frac{L_j}{3}\enspace,
  \end{align}
  where we use the slight abuse of notation by denoting the one-way marginals
  $q_j : \cX^{3n} \rightarrow \conv(\cX)$
  and $q_j : (\{0, 1\}^p)^n \rightarrow [0, 1]^p$ in the same way.
\end{claim}

\begin{proof}
  Let $D \in (\{0, 1\}^p)^n$, and let $D^{aug} \in (\{0, 1\}^p)^{3n}$
  constructed by
  adding $2n$ rows of $1$'s at the end of $D$. Let $D_L^{aug} = \chi_L(D)$.
  We remark that
  \begin{align}
    \label{claim:qj-qjaug-eq1}
    q_j(D^{aug})
    = \frac{1}{3n} \sum_{i=1}^{3n} D^{aug}_{i,j}
    = \frac{1}{3} \left(\frac{1}{n} \sum_{i=1}^n D^{aug}_{i,j} \right)
    + \frac{1}{3n} \sum_{i=n+1}^{3n} 1
    = \frac{1}{3} q_j(D) + \frac{2}{3} \in [0, 1]\enspace.
  \end{align}
  Then, we link $q_j(D^{aug})$ with $q_j(D^{aug}_L)$:
  \begin{align}
    \label{claim:qj-qjaug-eq2}
    q_j(D^{aug}_L)
    = \frac{1}{3n} \sum_{i=1}^{3n} (D_L^{aug})_{i,j}
    = \frac{1}{3n} \sum_{i=1}^{3n} 2L_j((D^{aug})_{i,j} - 1/2)
    = 2L_j (q_j(D^{aug}) - 1/2) \in [-L_j, L_j]\enspace,
  \end{align}
  combining \eqref{claim:qj-qjaug-eq1} and \eqref{claim:qj-qjaug-eq2} gives the
  result.
\end{proof}

\paul{adapt def of counting queries}
\begin{theorem}
  \label{thm:lower-bound-one-way-marginals}
  Let $n, p \in \NN$, and $L_1, \dots, L_p > 0$.
  Assume that for all subsets $\cJ \subseteq [p]$ of size at least $\lceil \frac{p}{75} \rceil$,
  $\sqrt{\sum_{j\in\cJ} L_j^2} = \Omega(\norm{L}_2)$.
  Define $\cX = \prod_{j=1}^p \{-L_j; +L_j\}$, and let $q_j : \cX \rightarrow \{-L_j, L_j\}$
  be the predicate of the $j$-th one-way marginal on $\cX$.
  Take $\epsilon > 0$ and $\delta = o(\frac{1}{n})$.
  There exists a number $M = \Omega\left(\min\left(n \norm{L}_2, \frac{\sqrt{p} \norm{L}_2}{\epsilon}\right)\right)$
  such that for every $(\epsilon,\delta)$-differentially private algorithm $\cA$,
  there exists a dataset $D = \{d_1, \dots, d_n\} \in \cX^n$ with $\norm{\sum_{i=1}^n d_i}_2 \in [M-1, M+1]$
  such that, with probability at least $1/3$ over the randomness of $\cA$:
  \begin{align}
    \norm{\cA(D) - q(D)}_2 = \Omega\left(\min\left(\norm{L}_2, \frac{\sqrt{p} \norm{L}_2}{n \epsilon}\right)\right)\enspace.
  \end{align}
\end{theorem}

\begin{proof}
  Let $M = \Omega\left(\min\left(n \norm{L}_2, \frac{\sqrt{p} \norm{L}_2}{\epsilon}\right)\right)$,
  and define the set of queries $\cQ$ composed of $p$ queries $q_j(D) = \frac{1}{n}\sum_{i=1}^n d_{i,j}$
  for $j \in [p]$.
  Let $\cA$ be a $(\epsilon, \delta)$-differentially-private randomized
  algorithm.
  Let $\alpha, \beta \in [0, 1]$.
  We will show that there exists a dataset $D$ such that
  $\norm{\sum_{i=1}^n d_i}_2 \in [M-1, M+1]$ for which $\cA(D)$ is not
  $(\alpha, \beta)$-accurate.

  \paragraph{When \boldmath$n \le n^*$.}
  Assume, for the sake of contradiction, that $\cA : \cX^{3n} \rightarrow \conv(\cX)$
  is $(\tfrac{1}{3}\alpha, \beta)$-accurate for $\cQ$.
  Then, for each dataset $D' \in \cX^{3n}$, we have
  \begin{align}
    \label{thm:lower-bound-one-way-marginals:accuracy-A}
    \prob{ \exists \cJ \subseteq [p]~\text{such that }\card{\cJ}
      \ge (1 - \beta)p \text{ and }\forall j \in \cJ,~
      \abs{\cA_j(D') - q_j(D')} < \frac{2L_j}{3} \alpha }
    \ge 2/3\enspace.
  \end{align}
  Importantly, for all $D \in (\{0,1\})^p)^n$, the randomized algorithm $\cA$
  satisfies~\eqref{thm:lower-bound-one-way-marginals:accuracy-A} for the
  dataset $D_L^{aug} = \chi_L(D) \in \cX^{3n}$.
  We now construct the mechanism $\widetilde \cA : (\{0, 1\}^p)^n \rightarrow [0, 1]^p$
  that takes a dataset $D \in (\{0, 1\}^p)^n$, constructs $D_L^{aug} = \chi_L(D)$
  and runs $\cA$ on it.
  It then outputs $\widetilde \cA(D)$ such that, for $j \in [p]$,
  $\widetilde \cA_j(D) = \frac{3}{2L_j} \cA_j(D_L^{aug}) - \frac{L_j}{3}$.
  Using \Cref{claim:qj-qjaug}, the results of $\widetilde \cA$ and be linked to
  the ones of $\cA$, as
  \begin{align}
    \label{thm:lower-bound-one-way-marginals:rewrite-wdA}
    \abs{\widetilde \cA(D) - q_j(D)}
    = \abs{\frac{3}{2L_j} \cA_j(D_L^{aug})- \frac{L_j}{3} - \frac{3}{2L_j} q_j(D_L^{aug}) + \frac{L_j}{3}}
    = \frac{3}{2L_j} \abs{\cA_j(D_L^{aug}) - q_j(D_L^{aug})}\enspace.
  \end{align}

  Therefore, if $\cA$ satisfies~\eqref{thm:lower-bound-one-way-marginals:accuracy-A}
  and \eqref{thm:lower-bound-one-way-marginals:rewrite-wdA}, then
  $\widetilde \cA : (\{0, 1\}^p)^n \rightarrow [0, 1]^p$ satisfies,
  for all $D \in (\{0, 1\}^p)^n$,
  \begin{align}
    \label{thm:lower-bound-one-way-marginals:accuracy-widetildeA}
    \prob{ \exists \cJ \subseteq [p]~\text{such that }\card{\cJ}
      \ge (1 - \beta)p \text{ and }\forall j \in \cJ,~
      \abs{\widetilde \cA_j(D) - q_j(D)} < \alpha }
    \ge 2/3\enspace,
  \end{align}
  which is exactly the definition of $(\alpha, \beta)$-accuracy for $\widetilde\cA$.
  Remark that since $\widetilde\cA$ is only a post-processing of $\cA$, without
  additional access to the dataset itself, $\widetilde\cA$ is itself
  $(\epsilon, \delta)$-differentially-private.
  We have thus constructed an algorithm that is both accurate and private for
  $n \le n^*$, which contradicts the result of
  \Cref{lemma:bun-fingerprint-1way-marginal} when $\beta = \frac{1}{75}$.
  This proves the existence of a dataset $D \in (\{0, 1\}^p)^n$ such that
  for $D_L^{aug} = \chi_L(D)$, $\cA(D_L^{aug})$ is not
  $(\tfrac{1}{3}\alpha, \beta)$-accurate on $\cQ$, which means that with
  probability at least $1/3$, there exists a subset $\cJ \subseteq [p]$ of
  cardinal $\card{\cJ} \ge \lceil \beta p \rceil$ such that
  \begin{align}
    \label{eq:thm:lower-bound-one-way-marginals:upper-bound-normL}
    \norm{\cA(D_L^{aug}) - q(D_L^{aug})}_2
    \overset{\eqref{thm:lower-bound-one-way-marginals:accuracy-A}}{\ge}
    \sqrt{\sum_{j \in \cJ} \frac{4L_j^2}{9}}
    \ge \Omega( \norm{L}_2 )\enspace,
  \end{align}
  where the second inequality comes from the fact that
  $\card{\cJ} \ge  \lceil \beta p \rceil = \lceil \frac{p}{75} \rceil$ and
  our hypothesis on $\sum_{j\in\cJ} L_j^2$.
  Notice that when $L_1 = \cdots = L_p = \frac{1}{\sqrt{p}}$, we recover the
  result of \citet{bassily2014Private}, since $\norm{L}_2 = 1$ it holds
  with probability at least $1/3$ that
  \begin{align}
    \norm{\cA(D_L^{aug}) - q(D_L^{aug})}_2
    \overset{\eqref{thm:lower-bound-one-way-marginals:accuracy-A}}{\ge}
    \sqrt{\sum_{j \in \cJ} \frac{4L_j^2}{9}}
    \ge \sqrt{\frac{4}{9 \times 75}} \norm{L}_2
    \ge \frac{2}{27}\enspace,
  \end{align}
  and in that case, since all $L_j$'s are equal, it indeed holds that
  $\sqrt{\sum_{j \in \cJ} L_j^2} = \Omega(\norm{L}_2)$.
  Finally, we remark that the sum of each column of $D_L^{aug}$ is
  $\sum_{i=1}^n d_{i,j} \ge n L_j$, and as such, we have
  $\norm{\sum_{i=1}^n d_i}_2 = \sqrt{\sum_{j=1}^p (\sum_{i=1}^n d_{i,j})^2}
    \ge \sqrt{\sum_{j=1}^p n^2 L_j^2} = n \norm{L}_2$.

  \paragraph{When \boldmath$n > n^*$.}
  We get the result in that case by augmenting the dataset $D^*$ that we
  constructed in the first part of this proof.
  To do so, we follow the steps described by \citet{bassily2014Private} in the
  proof of their Lemma 5.1.
  The construction consists in choosing a vector $c \in \cX$, and adding
  $\lceil \frac{n-n^*}{2} \rceil$ rows with $c$, and
  $\lfloor \frac{n-n^*}{2} \rfloor$ rows with $-c$ to the dataset $D^*$.
  This results in a dataset $D'$ such that
  $\norm{\sum_{i=1}^n d_i} = \Omega(n^* \norm{L}_2) = \Omega(\frac{\sqrt{p} \norm{L}_2}{\epsilon})$,
  since the contributions of rows $-c$ and $c$ (almost) cancel out.
  The theorem follows from observing that $(\frac{n^*}{n} \alpha,
    \beta)$-accuracy
  on this augmented dataset implies $(\alpha, \beta)$-accuracy on the original
  dataset.
  As such, if an algorithm is both private and $(\frac{n^*}{n} \alpha, \beta)$-accurate
  on the dataset $D'$, we get a contradiction, which gives the theorem as
  $\frac{n^*}{n} = \frac{\sqrt{p}}{n\epsilon}$.
\end{proof}

\begin{remark}
  \label{rmq:lower-bound-one-way-marginals-no-hyp-Lj}
  Without the assumption on the distribution of the $L_j$'s, we can still get
  an inequality that resembles~\eqref{eq:thm:lower-bound-one-way-marginals:upper-bound-normL}:
  $\norm{\cA(D_L^{aug}) - q(D_L^{aug})}_2
    \overset{\eqref{thm:lower-bound-one-way-marginals:accuracy-A}}{\ge}
    \sqrt{\sum_{j \in \cJ} \frac{4L_j^2}{9}}
    \ge \frac{2}{27} \frac{L_{\min}}{L_{\max}} \norm{L}_2$,
  with probability at least $1/3$, and we get a result similar to
  \Cref{thm:lower-bound-one-way-marginals}, except
  with an additional multiplicative factor $L_{\min}/L_{\max}$.
\end{remark}

\subsection{Lower Bound for Convex Functions}
\label{sec:convex-functions}

To prove a lower bound for our problem in the convex case, we let $L_1,
  \cdots, L_p > 0$ and
define a dataset $D = \{d_1, \dots, d_n\}$ taking its values in a set
$\cX = \prod_{j=1}^p \{\pm L_j\}$.
For $\beta > 0$, we consider the problem
\eqref{eq:dp-erm}
with the convex, smooth and $L$-component-Lipschitz loss function
$\ell(w; d) = - \scalar{w}{d}$ and the convex, separable regularizer
$\psi(w) = \frac{\norm{\sum_{i=1}^n d_i}_2}{\beta n} \norm{w}_2^2$:
\begin{align}
  \label{eq:dp-erm-lower-bound-convex}
  w^* = \argmin_{w\in\RR^p} \left\{
  F(w; D) = - \frac{1}{n}\scalar{w}{\textstyle{\sum_{i=1}^n d_i}}
  + \frac{\norm{\sum_{i=1}^n d_i}_2}{\beta n} \norm{w}_2^2 \right\}\enspace,
\end{align}
To find the solution of \eqref{eq:dp-erm-lower-bound-convex}, we look for $w^*$
so that the objective's gradient is zero, that is
\begin{align}
  \label{eq:dp-erm-lower-bound-convex-grad-obj-zero}
  w^* = \frac{\beta}{\norm{\sum_{i=1}^n d_i}_2} \sum_{i=1}^n d_i\enspace,
\end{align}
so that $\norm{w^*}_2 =  \frac{\beta}{\norm{\sum_{i=1}^n d_i}_2} \norm{\sum_{i=1}^n d_i}_2 = \beta$.
To prove the lower bound, we remark that
\begin{align}
  F(w; D) - F(w^*; D)
   & = - \frac{1}{n} \scalar{w - w^*}{\textstyle{\sum_{i=1}^n d_i}}
  + \frac{\norm{\textstyle{\sum_{i=1}^n d_i}}}{2 \beta n} (\norm{w}^2_2 - \norm{w^*}_2^2)      \\
   & = - \frac{1}{n} \scalar{w - w^*}{\frac{\norm{\textstyle{\sum_{i=1}^n d_i}}}{\beta} w^*}
  + \frac{\norm{\textstyle{\sum_{i=1}^n d_i}}}{2\beta n} (\norm{w}^2_2 - \norm{w^*}_2^2)       \\
   & = \frac{\norm{\textstyle{\textstyle{\sum_{i=1}^n d_i}}}}{\beta n}
  \left( \scalar{w^* - w}{w^*} + \frac{1}{2} \norm{w}_2^2 - \frac{1}{2} \norm{w^*}_2^2 \right) \\
   & = \frac{\norm{\textstyle{\textstyle{\sum_{i=1}^n d_i}}}}{\beta n}
  \left( - \scalar{w}{w^*} + \frac{1}{2} \norm{w}_2^2 + \frac{1}{2} \norm{w^*}_2^2 \right)     \\
   & = \frac{\norm{\textstyle{\textstyle{\sum_{i=1}^n d_i}}}}{2 \beta n}
  \norm{w - w^*}_2^2\enspace.
\end{align}

At this point, we can proceed similarly to \citet{bassily2014Private}
to relate this quantity to private estimation of one-way marginals.
We let $M = \Omega(\min(n\norm{L}_2, \norm{L}_2\sqrt{p}/\epsilon))$
and $\cA$ be an $(\epsilon,\delta)$-differentially private mechanism
that outputs a private solution $w^{priv}$ to
\eqref{eq:dp-erm-lower-bound-convex}.  Suppose, for the sake of
contradiction, that for every dataset $D$ with
$\norm{\textstyle{\sum_{i=1}^n d_i}}_2 \in [M-1; M+1]$, it holds with
probability at least $2/3$ that
\begin{align}
  \label{lower-bound-convex-function-inequality-beta}
  \norm{w^{priv} - w^*} \neq \Omega(\beta)\enspace.
\end{align}

We now derive from $\cA$ a mechanism $\widetilde \cA$ to estimate one-way
marginals. To do this, $\widetilde \cA$ runs $\cA$ to obtain $w^{priv}$ and
outputs $\frac{M}{n\beta} w^{priv}$.
We obtain that with probability at least $2/3$,
\begin{align}
  \label{lower-bound-convex-function-inequality-contradiction-with-beta}
  \norm{\widetilde \cA(D) - q(D)}_2
  = \frac{M}{n\beta} \norm{w^{priv} - \frac{\beta}{M} \textstyle{\sum_{i=1}^n d_i}}_2
  \not= \Omega\left(\frac{M}{n}\right)
  = \Omega\left(\min\left(\norm{L}_2, \frac{\norm{L}_2\sqrt{p}}{n\epsilon}\right)\right)\enspace.
\end{align}
where $q(D) = \frac{1}{n} \sum_{i=1}^n d_i$.
This is in contradiction with \Cref{thm:lower-bound-one-way-marginals}.
We thus proved that $\norm{w^{priv} - w^*} = \Omega(\beta)$, with probability at
least $1/3$.
As a consequence, we now obtain that with probability at least $1/3$,
\begin{align}
  F(w^{priv}; D) - F(w^*; D)
  = \frac{\norm{\textstyle{\textstyle{\sum_{i=1}^n d_i}}}}{2 \beta n} \norm{w^{priv} - w^*}_2^2
  = \Omega\left(\min\left(\norm{L}_2\beta, \frac{\beta\norm{L}_2\sqrt{p}}
    {n\epsilon}\right)\right)\enspace,
\end{align}
which gives the desired result on the expectation of $F(w^{priv}; D) - F(w^*; D)$.

Finally, if we do not make any hypothesis on the $L_j$'s distribution, we
can directly use
the non-augmented dataset constructed by \citet{bun2014Fingerprinting} to prove
\Cref{lemma:bun-fingerprint-1way-marginal} (that is the dataset from
\Cref{thm:lower-bound-one-way-marginals}, rescaled but not augmented).
The $\ell_2$-norm of the sum of this dataset is
$\norm{\sum_{i=1}^n d_j}_2 = [M' - 1, M' + 1]$ with
$M' = \Omega\left(\min\left(\frac{L_{\min}}{L_{\max}} n \norm{L}_2, \frac{L_{\min}}{L_{\max}} \frac{\sqrt{p}\norm{L}_2}{\epsilon} \right)\right)$.
This holds since four columns of this dataset out of five have sum of
$\pm n L_j$ (for some $j$'s), but no lower bound on the sum of the remaining
columns can be derived.
Thus, assuming~\eqref{lower-bound-convex-function-inequality-beta} holds,
then~\eqref{lower-bound-convex-function-inequality-contradiction-with-beta} can
be rewritten as
\begin{align}
  \norm{\widetilde \cA(D) - q(D)}_2
  = \frac{M'}{n\beta} \norm{w^{priv} - \frac{\beta}{M} \textstyle{\sum_{i=1}^n d_i}}_2
  \not= \Omega\left(\frac{M'}{n}\right)
  = \Omega\left(\min\left(\frac{L_{\min}}{L_{\max}} \norm{L}_2, \frac{L_{\min}}{L_{\max}} \frac{\norm{L}_2\sqrt{p}}{n\epsilon}\right)\right)\enspace,
\end{align}
with probability at least $1/3$, which is in contradiction with
Remark~\ref{rmq:lower-bound-one-way-marginals-no-hyp-Lj}.
We thus get an additional factor of $L_{\min}/L_{\max}$ in the lower bound:
\begin{align}
  F(w^{priv}; D) - F(w^*; D)
  = \frac{\norm{\textstyle{\textstyle{\sum_{i=1}^n d_i}}}}{2 \beta n} \norm{w^{priv} - w^*}_2^2
  = \Omega\left(\min\left(\frac{L_{\min}}{L_{\max}}\norm{L}_2\beta, \frac{L_{\min}}{L_{\max}}\frac{\beta\norm{L}_2\sqrt{p}}{n\epsilon}\right)\right)\enspace.
\end{align}

\subsection{Lower Bound for Strongly-Convex Functions}
\label{sec:strongly-conv-funct}

To prove a lower bound for strongly-convex functions, we let $\mu_I > 0$,
$L_1, \dots, L_p > 0$, $\cW = \prod_{j=1}^p [-\frac{L_j}{2\mu_I}, +\frac{L_j}{2\mu_I}]$ and
$D = \{ d_1, \dots, d_n \} \in \prod_{j=1}^p \{\pm \frac{L_j}{2\mu_I} \}$.
We consider the following problem, which fits in our setting:
\begin{align}
  \label{eq:dp-erm-lower-bound-strongly-convex}
  w^* = \argmin_{w\in\RR^p} \left\{
  F(w;D) = \frac{\mu_I}{2n} \sum_{i=1}^n \norm{w - d_i}_2^2
  + i_{\cW}(w) \right\}\enspace.
\end{align}
where $i_{\cW}$ is the (separable) characteristic function of the set
$\cW$. Since $\psi = i_{\cW}$ is the characteristic function of a box-set, the
proximal operator is equal to the projection on $\cW$ and DP-CD iterates
are thus guaranteed to remain in $\cW$. Therefore, regularity
assumptions on $f$ only need to hold on $\cW$, as pointed out in \Cref{rmq:constrained-regularity-assumptions}.  The loss
function $\ell(w;d_i)=\frac{\mu_I}{2} \norm{w - d_i}_2^2$ is
$L$-component-Lipschitz on $\cW$ since, for $w \in \cW$ and $j \in [p]$,
the triangle inequality gives:
\begin{align}
  \abs{\nabla_j \ell(w;d_i)}
  \le \mu_I (\abs{w_j} + \abs{d_{i,j}})
  \le \mu_I \left(\frac{L_j}{2\mu_I} + \frac{L_j}{2\mu_I}\right)
  \le L_j\enspace.
\end{align}
This loss is also $\mu_I$-strongly convex \wrt $\ell_2$-norm since
for $w, w' \in \cW$,
\begin{align}
  \ell(w; d_i)
  = \frac{\mu_I}{2} \norm{w - d_i}_2^2
  = \frac{\mu_I}{2} \norm{w' - d_i + w - w'}_2^2
  = \frac{\mu_I}{2} \left(\norm{w' - d_i}_2^2 + 2\scalar{w' - d_i}{w - w'} + \norm{w - w'}_2^2\right)\enspace,
\end{align}
which is exactly $\mu_I$-strong convexity since
$\ell(w';d_i) = \frac{\mu_I}{2} \norm{w' - d_i}_2^2$ and
$\nabla \ell(w'; d_i) = \mu_I(w - d_i)$.
The minimum of the objective function in~\eqref
{eq:dp-erm-lower-bound-strongly-convex} is attained
at $w^* = \frac{1}{n} \sum_{i=1}^n d_i = q(D) \in \cW$.
The excess risk of $F$ is thus
\begin{align}
  F(w; D) - F(w^*)
   & = \frac{\mu_I}{2n} \sum_{i=1}^n \norm{w - d_i}_2^2 - \norm{w^* - d_i}_2^2           \\
   & = \frac{\mu_I}{2n} \sum_{i=1}^n \norm{w}^2 - \norm{w^*}^2 + 2 \scalar{d_i}{w^* - w} \\
   & = \frac{\mu_I}{2}\norm{w}^2 - \frac{1}{2}\norm{w^*}^2 + \scalar{w^*}{w^* - w}       \\
   & = \frac{\mu_I}{2} \norm{w - q(D)}_2^2\enspace.
\end{align}

It remains to apply \Cref{thm:lower-bound-one-way-marginals} to obtain that, with
probability at least $1/3$,
\begin{align}
  F(w^{priv}; D) - F(w^*)
   & = \Omega\left( \min\left(\frac{\norm{L}_2^2}{\mu_I}, \frac{\norm{L}_2^2 p}{\mu_I n^2\epsilon^2}\right)\right)\enspace,
\end{align}
which gives the lower bound on the expected value of $F(w^{priv}; D) - F(w^*)$.
Note that without the additional assumption on the distribution of the $L_j$'s,
Remark~\ref{rmq:lower-bound-one-way-marginals-no-hyp-Lj} directly gives the
result with an additional multiplicative factor $(L_{\min} / L_{\max})^2$:
\begin{align}
  F(w^{priv}; D) - F(w^*)
   & = \Omega\left( \min\left(\frac{L_{\min}^2}{L_{\max}^2}\frac{\norm{L}_2^2}{\mu_I},
    \frac{L_{\min}^2}{L_{\max}^2}\frac{\norm{L}_2^2 p}{\mu_I n^2\epsilon^2}\right)\right)\enspace,
\end{align}
with probability at least $1/3$.

\section{Private Estimation of Smoothness Constants}
\label{sec:priv-estim-smoothn}

In this section, we explain how a fraction $\epsilon'$ of the $\epsilon$ budget
of DP
can be used to estimate the coordinate-wise smoothness constants,
which are essential to the good performance of DP-CD on imbalanced problems.
Let
$f$ be
defined as the average loss over the dataset $D$ as in
problem~\eqref{eq:dp-erm}. We
denote
by $M_j^{(i)}$ the $j$-th component-smoothness constant of
$\ell(\cdot, d_i)$, where $d_i$ is the $i$-th point in
$D$. The $j$-th smoothness constant of the function $f$ is thus the
average of all these constants:
$M_j = \frac{1}{n} \sum_{i=1}^n M_j^{(i)}$.

Assuming that the practitioner knows an approximate upper bound $b_j$
over the $M_j^{(i)}$'s, they can enforce it by clipping $M_j^{(i)}$ to
$b_j$ for each $i\in[n]$. The sensitivity of the average of the clipped $M_j^{
(i)}$'s is
thus $2b_j/n$. One can then compute an estimate of $M_1,\dots,M_p$ under $\epsilon$-DP using the Laplace mechanism as follows:
\begin{align}
  M_j^{priv} = \frac{1}{n} \sum_{i=1}^n \clip(M_j^{(i)}, b_j) + \text{Lap}\left
  (\frac{2b_jp}{n\epsilon'}\right),\quad\text{for each }j\in[p]\enspace,
\end{align}
where the factor $p$ in noise scale comes from using the simple
composition theorem \cite{dwork2013Algorithmic}, and
$\text{Lap}(\lambda)$ is a sample drawn in a Laplace distribution of
mean zero and scale $\lambda$.  The computed constant can then
directly be used in DP-CD, allocating the remaining budget
$\epsilon-\epsilon'$ to the optimization procedure.

\section{Additional Experimental Details and Results}
\label{sec:experimental-setup}

\subsection{Hyperparameter Tuning}
\label{sec:hyperp-tuning}

DP-SGD and DP-CD both depend on three hyperparameters: step size,
clipping threshold and number of passes on data. For DP-CD, step sizes
are adapted from a parameter as described in
\Cref{sec:numerical-experiments}, and clipping thresholds as well (see \Cref
{sub:clipping}). For DP-SGD, the step size is given by $\gamma/\beta$,
where $\gamma$ is the hyperparameter and $\beta$ is the problem's global
smoothness constant  (which we consider given), and the clipping threshold is
used directly to clip gradients along their $\ell_2$-norm.

We simultaneously tune these three hyperparameters for each algorithm across
the following grid:
\begin{itemize}
\item step size: 10 logarithmically-spaced values between $10^{-6}$
  and $1$ for DP-SGD, and between $10^{-2}$ and $10$ for DP-CD.\footnote{Recall that step sizes for CD algorithms are coordinate-wise, and
  thus larger than in SGD algorithms. We empirically verify that
  the best
  step size always lies strictly inside the considered interval for both
  DP-CD and DP-SGD.}
\item clipping threshold: 100 logarithmically-spaced values, between
  $10^{-3}$ and $10^{6}$.
\item number of passes: 5 values (2, 5, 10, 20 and 50).
\end{itemize}
We run each algorithm on each dataset 5~times on each combination of
hyperparameter values.  We then keep the set of hyperparameters that
yield the lowest value of the objective at the last iterate, averaged
across the $5$~runs.

In \Cref{tab:full-tuning-max-iter}, we report the best relative error
(in comparison to optimal objective value) at the last iterate,
averaged over five runs, for each dataset, algorithm, and total number
of passes on the data. As such, each cell of this table corresponds to
the best value obtained after tuning the step size and clipping
hyperparameters for a given number of passes.

\begin{table}[t]
  \centering
  \scriptsize

  \caption{
    Relative error to non-private optimal value of the
    objective function for different number of passes on the
    data. Results are reported for each dataset and for DP-CD and
    DP-SGD, after tuning step size and clipping
    hyperparameters. A star indicates the lowest error in each row.
  }

  \label{tab:full-tuning-max-iter}
  \begin{tabular}{ccccccc}
    \toprule
    & Passes on data &2  	 &5  	 &10  	 &20  	 &50  	 \\
    \midrule
    Electricity (imbalanced) & DP-CD  & $ 0.1458 \pm 6\text{e-}04 $ 	 & $ 0.0842 \pm 1\text{e-}03 $ 	 & $ 0.0436 \pm 2\text{e-}03 $ 	 & $ 0.0147 \pm 2\text{e-}03 $ 	 & $ 0.0020 \pm 1\text{e-}03 $* 	 \\
    $\epsilon=1, \delta=1/n^2$ & DP-SGD  & $ 0.2047 \pm 2\text{e-}02 $ 	 & $ 0.1804 \pm 2\text{e-}02 $ 	 & $ 0.1766 \pm 2\text{e-}02 $ 	 & $ 0.1644 \pm 2\text{e-}02 $ 	 & $ 0.1484 \pm 1\text{e-}02 $* 	 \\
    \midrule
    Electricity (balanced) & DP-CD  & $ 0.0186 \pm 4\text{e-}04 $ 	 & $ 0.0023 \pm 4\text{e-}04 $ 	 & $ 0.0013 \pm 6\text{e-}04 $* 	 & $ 0.0013 \pm 4\text{e-}04 $ 	 & $ 0.0019 \pm 8\text{e-}04 $ 	 \\
    $\epsilon=1, \delta=1/n^2$ & DP-SGD  & $ 0.0391 \pm 1\text{e-}02 $ 	 & $ 0.0189 \pm 5\text{e-}03 $ 	 & $ 0.0123 \pm 4\text{e-}03 $ 	 & $ 0.0106 \pm 3\text{e-}03 $ 	 & $ 0.0040 \pm 2\text{e-}03 $* 	 \\
    \midrule
    California (imbalanced) & DP-CD  & $ 0.1708 \pm 7\text{e-}03 $ 	 & $ 0.1232 \pm 1\text{e-}02 $ 	 & $ 0.0598 \pm 1\text{e-}02 $ 	 & $ 0.0287 \pm 5\text{e-}03 $ 	 & $ 0.0124 \pm 7\text{e-}03 $* 	 \\
    $\epsilon=1, \delta=1/n^2$ & DP-SGD  & $ 0.2799 \pm 9\text{e-}02 $ 	 & $ 0.1863 \pm 2\text{e-}02 $ 	 & $ 0.1476 \pm 2\text{e-}02 $ 	 & $ 0.1094 \pm 2\text{e-}02 $ 	 & $ 0.1068 \pm 2\text{e-}02 $* 	 \\
    \midrule
    California (balanced) & DP-CD  & $ 0.0007 \pm 3\text{e-}04 $* 	 & $ 0.0011 \pm 6\text{e-}04 $ 	 & $ 0.0012 \pm 5\text{e-}04 $ 	 & $ 0.0010 \pm 1\text{e-}04 $ 	 & $ 0.0017 \pm 1\text{e-}03 $ 	 \\
    $\epsilon=1, \delta=1/n^2$ & DP-SGD  & $ 0.0351 \pm 2\text{e-}02 $ 	 & $ 0.0226 \pm 8\text{e-}03 $ 	 & $ 0.0125 \pm 3\text{e-}03 $ 	 & $ 0.0087 \pm 2\text{e-}03 $ 	 & $ 0.0042 \pm 1\text{e-}03 $* 	 \\
    \midrule
    Sparse LASSO & DP-CD  & $ 0.2498 \pm 4\text{e-}02 $* 	 & $ 0.4702 \pm 9\text{e-}02 $ 	 & $ 0.5982 \pm 4\text{e-}02 $ 	 & $ 0.7160 \pm 2\text{e-}02 $ 	 & $ 0.7551 \pm 0\text{e+}00 $ 	 \\
    $\epsilon=10, \delta=1/n^2$ & DP-SGD  & $ 0.7551 \pm 0\text{e+}00 $ 	 & $ 0.7551 \pm 3\text{e-}09 $* 	 & $ 0.7551 \pm 0\text{e+}00 $ 	 & $ 0.7551 \pm 0\text{e+}00 $ 	 & $ 0.7551 \pm 0\text{e+}00 $ 	 \\
    \bottomrule

  \end{tabular}
\end{table}

\subsection{Running Time}
\label{sec:running-time-in}

\begin{figure}[t]
  \captionsetup[subfigure]{justification=centering}
  \centering
  \begin{subfigure}{0.045\linewidth}
    \centering
    \includegraphics[width=\linewidth]{plots/xlegend.pdf}
    \begin{minipage}{.1cm}
      \vfill
    \end{minipage}
  \end{subfigure}%
  \begin{subfigure}{0.3\linewidth}
    \centering
    \includegraphics[width=\linewidth]{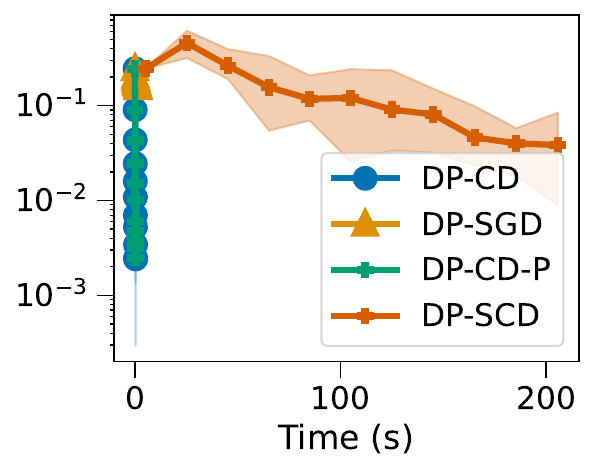}
    \caption{Electricity (logistic).\\ Imbalanced.}
    \label{fig:expe-time-electricity-raw}
  \end{subfigure}%
  \begin{subfigure}{0.3\linewidth}
    \centering
    \includegraphics[width=\linewidth]{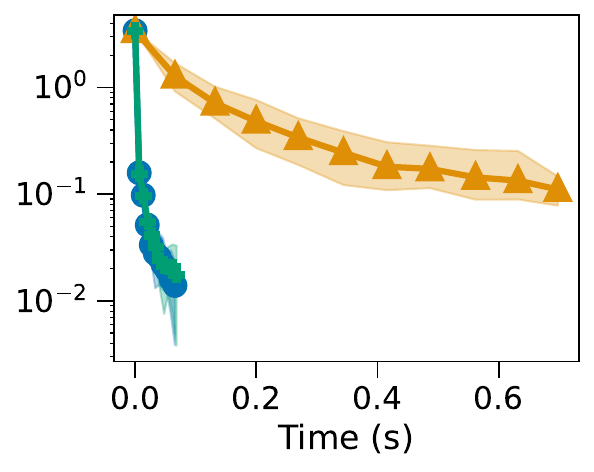}
    \caption{California (LASSO).\\ Imbalanced.}
    \label{fig:expe-time-california-raw}
  \end{subfigure}
  \begin{subfigure}{0.045\linewidth}
    \centering
    \includegraphics[width=\linewidth]{plots/xlegend.pdf}
    \begin{minipage}{.1cm}
      \vfill
    \end{minipage}
  \end{subfigure}%
  \begin{subfigure}{0.3\linewidth}
    \centering
    \includegraphics[width=\linewidth]{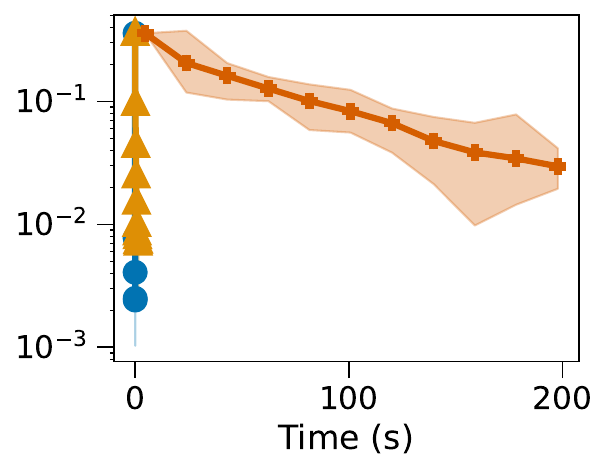}
    \caption{Electricity (logistic).\\ Balanced.}
    \label{fig:expe-time-electricity-norm}
  \end{subfigure}%
  \begin{subfigure}{0.3\linewidth}
    \centering
    \includegraphics[width=\linewidth]{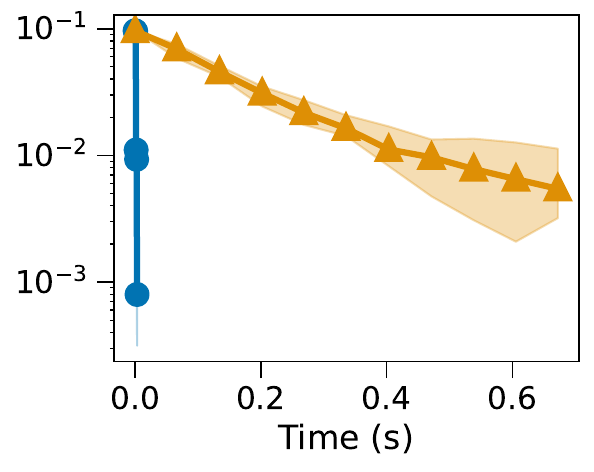}
    \caption{California (LASSO).\\ Balanced.}
    \label{fig:expe-time-california-norm}
  \end{subfigure}
  \begin{subfigure}{0.3\linewidth}
    \centering
    \includegraphics[width=\linewidth]{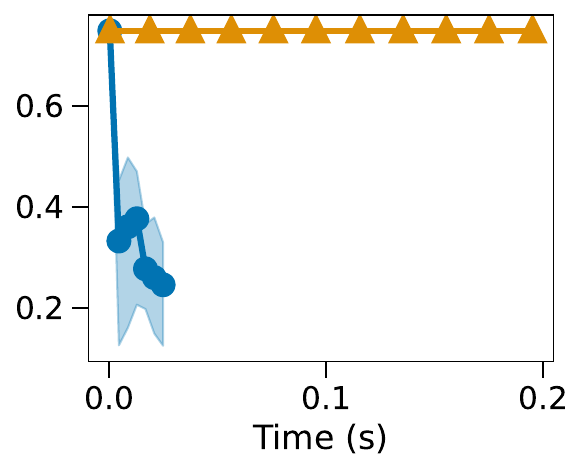}
    \caption{Sparse LASSO.\\ Balanced.}
    \label{fig:expe-time-lasso}
  \end{subfigure}

  \caption{Relative error to non-private optimal for DP-CD (blue,
    round marks), DP-CD with privately estimated coordinate-wise smoothness
    constants (green, + marks) and DP-SGD (orange, triangle marks) on
    five problems. We report average, minimum and maximum values over
    10~runs for each algorithm, as a function of the algorithm running
    time (in seconds).  }
  \label{fig:expe-time}
\end{figure}

\aurelien{say something about the implementation (and recall it is included
as supplementary material), and the machine on which the experiments are run.}

In this section, we report the running times of DP-CD and DP-SGD.  We
implemented DP-CD and DP-SGD in C++, with Python bindings\footnote{The
  code is available at
  \url{https://gitlab.inria.fr/pmangold1/private-coordinate-descent/}.}. The
design matrix and the labels are kept in memory as dense matrices of
the Eigen library. No special code optimization nor tricks is applied
to the algorithms, except for the update of residuals at each
iteration of DP-CD, which prevents from accessing the complete dataset
at each step. All experiments were run on a laptop with 16GB of RAM
and an Intel(R) Core(TM) i7-10610U CPU @ 1.80GHz.

\Cref{fig:expe-time} shows the same experiments as in
\Cref{fig:expe-raw} and \Cref{fig:expe-standardized}, but as a
function of the running time. In our implementation, DP-CD runs about
$4$ times as fast as DP-SGD for a given number of iterations (see
\Cref{fig:expe-time-electricity-raw} and
\Cref{fig:expe-time-california-raw} for $50$ iterations). On the three
other plots, \Cref{fig:expe-time-electricity-norm},
\Cref{fig:expe-time-california-norm} and \Cref{fig:expe-time-lasso},
DP-CD yields better results in less iterations. DP-CD is thus
particularly valuable in these scenarios: combined with its faster
running time, it provides accurate results extremely fast.  For
completeness, we provide in \Cref{tab:full-tuning-runtime} the full
table of running time, corresponding to
\Cref{tab:full-tuning-max-iter} and \Cref{fig:expe-time}. These
results show that, for a given number of passes on the data, DP-CD
consistently runs about $5$ times faster than DP-SGD.

\begin{table}[H]
  \centering
  \scriptsize

  \caption{ Time of execution (in seconds) for different number of
    passes on the data (averaged over 10~runs). Results are reported
    for each dataset and for DP-CD and DP-SGD, after tuning step size
    and clipping hyperparameters.
  }

  \label{tab:full-tuning-runtime}
  \begin{tabular}{ccccccc}
    \toprule
    & Passes on data &2  	 &5  	 &10  	 &20  	 &50  	 \\
    \midrule
    Electricity (imbalanced) & DP-CD & $ 0.0128 \pm 1\text{e-}03 $ 	 & $ 0.0274 \pm 1\text{e-}03 $ 	 & $ 0.0500 \pm 1\text{e-}03 $ 	 & $ 0.0980 \pm 7\text{e-}04 $ 	 & $ 0.2457 \pm 2\text{e-}03 $ 	 \\
    $\epsilon = 1, \delta=1/n^2$ & DP-SGD & $ 0.0663 \pm 2\text{e-}03 $ 	 & $ 0.1722 \pm 1\text{e-}02 $ 	 & $ 0.3321 \pm 1\text{e-}02 $ 	 & $ 0.6729 \pm 1\text{e-}02 $ 	 & $ 1.8588 \pm 2\text{e-}01 $ 	 \\
    \midrule
    Electricity (balanced) & DP-CD & $ 0.0121 \pm 7\text{e-}04 $ 	 & $ 0.0281 \pm 3\text{e-}03 $ 	 & $ 0.0529 \pm 2\text{e-}03 $ 	 & $ 0.1062 \pm 6\text{e-}03 $ 	 & $ 0.2577 \pm 2\text{e-}03 $ 	 \\
    $\epsilon = 1, \delta=1/n^2$ & DP-SGD & $ 0.0686 \pm 4\text{e-}03 $ 	 & $ 0.1768 \pm 1\text{e-}02 $ 	 & $ 0.3578 \pm 2\text{e-}02 $ 	 & $ 0.6787 \pm 2\text{e-}02 $ 	 & $ 1.6766 \pm 2\text{e-}02 $ 	 \\
    \midrule
    California (imbalanced) & DP-CD & $ 0.0029 \pm 9\text{e-}05 $ 	 & $ 0.0065 \pm 8\text{e-}05 $ 	 & $ 0.0130 \pm 1\text{e-}04 $ 	 & $ 0.0258 \pm 1\text{e-}04 $ 	 & $ 0.0647 \pm 2\text{e-}04 $ 	 \\
    $\epsilon = 1, \delta=1/n^2$ & DP-SGD & $ 0.0269 \pm 1\text{e-}03 $ 	 & $ 0.0665 \pm 1\text{e-}03 $ 	 & $ 0.1318 \pm 2\text{e-}03 $ 	 & $ 0.2628 \pm 3\text{e-}03 $ 	 & $ 0.6476 \pm 8\text{e-}03 $ 	 \\
    \midrule
    California (balanced) & DP-CD & $ 0.0031 \pm 2\text{e-}04 $ 	 & $ 0.0065 \pm 2\text{e-}04 $ 	 & $ 0.0132 \pm 1\text{e-}04 $ 	 & $ 0.0262 \pm 2\text{e-}04 $ 	 & $ 0.0649 \pm 3\text{e-}04 $ 	 \\
    $\epsilon = 1, \delta=1/n^2$ & DP-SGD & $ 0.0261 \pm 7\text{e-}04 $ 	 & $ 0.0641 \pm 5\text{e-}04 $ 	 & $ 0.1295 \pm 2\text{e-}03 $ 	 & $ 0.2592 \pm 4\text{e-}03 $ 	 & $ 0.6469 \pm 7\text{e-}03 $ 	 \\
    \midrule
    Sparse LASSO & DP-CD & $ 0.0244 \pm 6\text{e-}04 $ 	 & $ 0.0760 \pm 6\text{e-}04 $ 	 & $ 0.1614 \pm 4\text{e-}03 $ 	 & $ 0.3213 \pm 5\text{e-}04 $ 	 & $ 0.6598 \pm 1\text{e-}02 $ 	 \\
    $\epsilon = 10, \delta=1/n^2$ & DP-SGD & $ 0.0718 \pm 3\text{e-}03 $ 	 & $ 0.1788 \pm 4\text{e-}03 $ 	 & $ 0.3654 \pm 7\text{e-}03 $ 	 & $ 0.7292 \pm 2\text{e-}02 $ 	 & $ 1.8110 \pm 3\text{e-}02 $ 	 \\
    \bottomrule
  \end{tabular}
\end{table}

\end{document}